\documentclass[envcountsame]{article}
\pdfoutput=1
\usepackage{wrapfig}
\usepackage{sidecap}
\usepackage{amssymb}
\usepackage{amsmath}
\allowdisplaybreaks[3] 
\usepackage{bbold} 

\usepackage{flafter} 
\PassOptionsToPackage{pdftex}{graphicx} 

\usepackage[pdftex]{graphicx}

\usepackage{fancyhdr} 
\usepackage{ifthen} 
\usepackage{layout} 
\usepackage{xspace}
\usepackage{afterpage}
\usepackage{url}
\usepackage{textcomp} 

\usepackage{wasysym} 

\usepackage[latin1]{inputenc} 

\def\ba#1\ea{\begin{align*}#1\end{align*}} 
\def\banum#1\eanum{\begin{align}#1\end{align}} 

\def\bi#1\ei{\begin{itemize}#1\end{itemize}} 

\ifx\lemma\undefined
\newtheorem{lemma}{Lemma} 
\newtheorem{proposition}[lemma]{Proposition} 

\newtheorem{theorem}[lemma]{Theorem} 
\newtheorem{corollary}[lemma]{Corollary}
\newtheorem{definition}[lemma]{Definition}

\else
\fi

\newcommand{\ulesqed}{\hfill\smiley}

\ifx\proof\undefined
\newenvironment{proof}{\par\noindent{\em Proof. }}{\ulesqed\\[2mm]}
\else
\fi

\newcommand{\RR}{\mathbb{R}}
 
\newcommand{\Nat}{\mathbb{N}}

\let\R\undefined 
\newcommand{\R}{\RR}

\newcommand{\Ccal}{\mathcal{C}}

\renewcommand{\epsilon}{\ensuremath{\varepsilon}}
\renewcommand{\phi}{\ensuremath{\varphi}}

\newcommand{\blobb}[1]{%
\begin{list}{$\bullet$}{%
\setlength{\topsep}{0cm}
\setlength{\leftmargin}{0.5cm}
}{\item #1}%
\end{list}
}

\wasyfamily
\DeclareFontShape{U}{wasy}{b}{n}{ <-10> ssub * wasy/m/n
<10> <10.95> <12> <14.4> <17.28> <20.74> <24.88>wasyb10 }{}

\DeclareFontShape{U}{wasy}{m}{n}{ <5> <6> <7> <8> <9> gen * wasy
<10> <10.95> <12> <14.4> <17.28> <20.74> <24.88> <35> <40> <50> <60> wasy10  }{}

\newenvironment{itemize*}{
\begin{itemize}
\setlength{\parskip}{0em}
\setlength{\topparskip}{0em}
}
{\end{itemize}}

\newenvironment{enumerate*}{
\begin{enumerate}
\setlength{\parskip}{0em}
\setlength{\topparskip}{0em}
}
{\end{enumerate}}

\newsavebox{\savestuff}
\newlength{\backitem}
\newenvironment{boxit}{\begin{lrbox}{\savestuff}
	\begin{minipage}[b]{\linewidth}}
{\end{minipage}\end{lrbox}\fbox{\usebox{\savestuff}}}

{\begin{it} \hfill\parbox{#1}}
{\hspace{3em}\end{it}}

\newcommand{\beq}{\begin{equation}}
\newcommand{\eeq}{\end{equation}}
\newcommand{\beqa}{\begin{eqnarray}}
\newcommand{\eeqa}{\end{eqnarray}}
\newcommand{\beqas}{\begin{eqnarray*}}
\newcommand{\eeqas}{\end{eqnarray*}}

\newcommand{\bit}{\begin{itemize}}
\newcommand{\eit}{\end{itemize}}
\newcommand{\bits}{\begin{itemize*}}
\newcommand{\eits}{\end{itemize*}}
\newcommand{\benum}{\begin{enumerate}}
\newcommand{\eenum}{\end{enumerate}}
\newcommand{\benums}{\begin{enumerate*}}
\newcommand{\eenums}{\end{enumerate*}}

\newcommand{\ktrue}{K}
\newcommand{\kalg}{K'}

\usepackage[round,comma]{natbib}
\usepackage{flafter} %
\PassOptionsToPackage{pdftex}{graphicx} %

\usepackage[pdftex]{graphicx}

\newlength{\minipagewidth}
\renewcommand{\phi}{\varphi}
\let\P\undefined
\newcommand{\P}{\mathbb{P}}

\newcommand{\mmppage}[1]{}
\newcommand{\initalg}{{\sc Pruned MinDiam}}  %
\newcommand{\commentout}[1]{}

\newtheorem{assumption}{Assumption}

\newcommand{\cenzero}{c^{<0>}}
\newcommand{\cenunu}{c^{<1>}}
\newcommand{\finv}{\Phi^{-1}}
\newcommand{\tilr}{\tilde{R}}
\newcommand{\tila}{\tilde{A}}
\newcommand{\clust}{{\Ccal}}

\begin{document}

\title{
How the initialization affects the stability of the $k$-means algorithm
}

\author{S{\'e}bastien Bubeck\\
Sequel Project, INRIA Lille\\
Lille, France\\
{\tt sebastien.bubeck@inria.fr}
\\ \\
Marina Meil\u{a}\\
University of Washington\\
Department of Statistics\\
Seattle, USA\\
{\tt mmp@stat.washington.edu}
\\ \\
Ulrike von Luxburg \\
Max Planck Institute for Biological Cybernetics\\
T{\"u}bingen, Germany\\
{\tt ulrike.luxburg@tuebingen.mpg.de}}

\maketitle

\begin{abstract}
  We investigate the role of the initialization for the stability of
  the $k$-means clustering algorithm. As opposed to other papers, we
  consider the actual $k$-means algorithm and do not ignore its
  property of getting stuck in local optima. We are
  interested in the actual clustering, not only in the costs of
  the solution. We analyze when different initializations lead to the
  same local optimum, and when they lead to different local
  optima. This enables us to prove that it is reasonable to select the number of
  clusters based on stability scores.
\end{abstract}

\section{Introduction}

\begin{figure}[t]
\begin{center}
\includegraphics[width=0.49\textwidth]{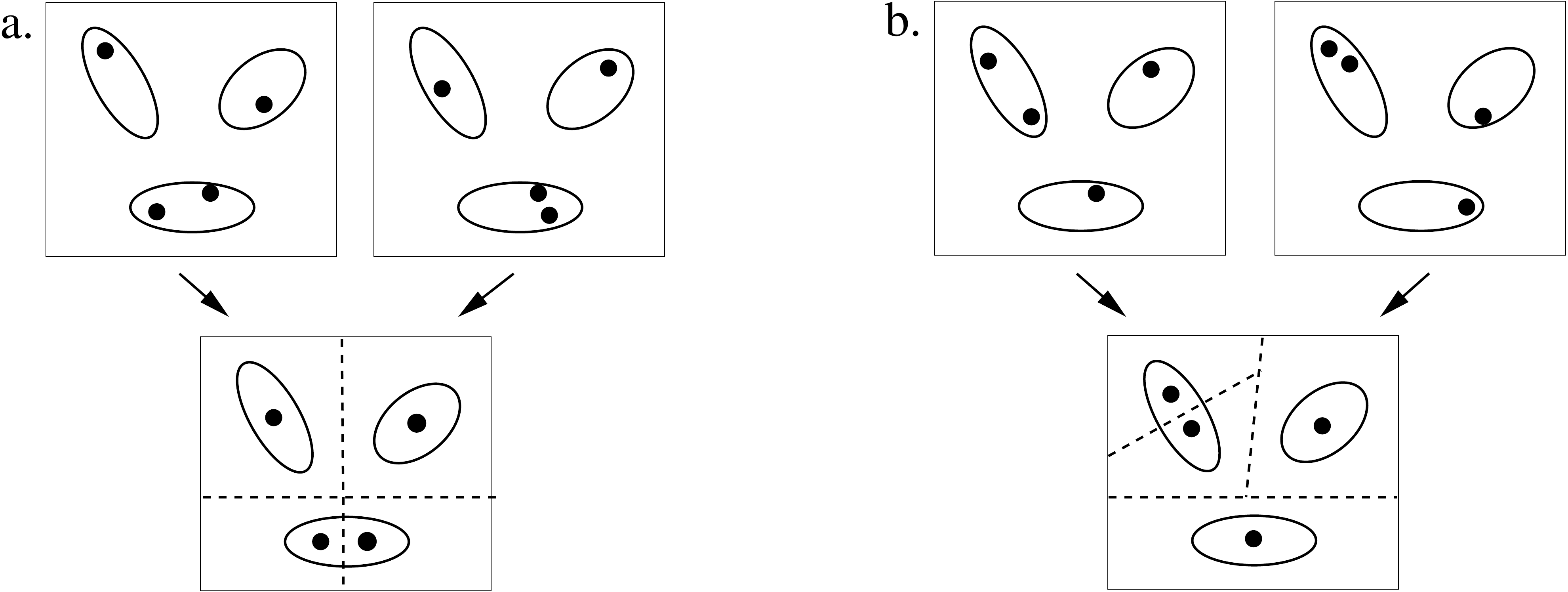}\hfill
\end{center}
\caption{Different initial configurations and the corresponding outcomes of the
  $k$-means algorithm. Figure a: the two boxes in the top row depict a data
  set with three clusters and four initial centers. Both boxes show
  different realizations of the same initial configuration. As can be
  seen in the bottom, both initializations lead to the same $k$-means
  clustering. Figure b: here the initial configuration is different
  from the one in Figure a, which leads to a different $k$-means
  clustering.
}
\label{fig-init-configuration}
\end{figure}

Stability is a popular tool for model selection in clustering, in
particular to select the number $k$ of clusters. The general idea is
that the best parameter $k$ for a given data set is the one which
leads to the ``most stable'' clustering results. While model selection
based on clustering stability is widely used in practice, its behavior
is still not well-understood from a theoretical point of view. A
recent line of papers discusses clustering stability with respect to
the $k$-means criterion in an idealized setting
\citep{BenLuxPal06,BenPalSim07,ShaTis08_nips,BenLux08,ShaTis08_colt,ShaTis09_nips}.
It is assumed that one has access to an ideal algorithm which can
globally optimize the $k$-means criterion.  For this perfect
algorithm, results on stability are proved in the limit of the sample
size $n$ tending to infinity. However, none of these results applies
to the $k$-means algorithm as used in practice: they do not take into
account the problem of getting stuck in local optima. In our current
paper we try to overcome this shortcoming. We
study the stability of the actual $k$-means {algorithm} rather than
the idealized one.  \\

Our analysis theoretically confirms the following intuition. Assume the data set has $\ktrue$ well-separated clusters, and assume
that $k$-means is initialized with $\kalg \geq \ktrue$
initial centers.  We conjecture that when there is at least
one initial center in each of the underlying clusters, then {\em the
initial centers tend to stay in the clusters they had been placed in. }

Consequently, the final clustering result is essentially determined by the {\em number} of initial centers in each
of the true clusters (which we call the {\em initial configuration}), see Figure~\ref{fig-init-configuration} for an illustration.
In particular if one uses an initialization scheme which has the desired property of placing at least one center in each cluster with high probability, then the following will
hold:
{If $K' = K$, we have one center per cluster, with
  high probability. The configuration will remain the same during the course of the algorithm.}
{If $K' > K$, different configurations can occur. Since different configurations lead to different
  clusterings we obtain significantly different final clusterings depending on the random initialization, in other words we observe {\em instability (w.r.t initialization)} . }

Note that our argument does not imply stability or instability for $K' < K$. As we have less initial centers than clusters, for any initialization scheme there will be some clusters with no initial center. In this setting centers do move between clusters, and this cannot be analyzed without looking at the actual positions of the centers. Actually, as can be seen from examples, in this case one can have either stability or instability. \\

The main point of our paper is that the arguments above can explain
why the parameter $k$ selected by stability based model selection is
often the true number of clusters, under the assumption that the data
set consists of well separated clusters and one uses an appropriate
initialization scheme. \\

Even though the arguments above are very intuitive, even individual parts of our conjecture turn out to be
surprisingly hard. In this paper we only go a first step towards a complete proof,
considering mixtures of Gaussians in one dimension. For a mixture of
two Gaussians ($\ktrue=2$) we prove that the $k$-means algorithm is
stable for $\kalg=2$ and instable for $\kalg=3$. The proof technique is
based on our configuration arguments outlined above. We also provide some preliminary results to study the general case, that is when the data space is $\R^d$ and we do not make any parametric assumption on the probability distribution.
Then we have a closer look at initialization schemes for $k$-means,
when $\kalg\geq\ktrue$. Is there an initialization scheme that will
place at least one center in each true cluster w.h.p? Clearly, the
naive method of sampling $\kalg$ centers from the data set does not
satisfy this property except for very small $K$. We study a standard
but not naive initialization scheme and prove that it has the
desirable property we were looking for. \\

Of course there exist numerous other papers which
study theoretical properties of the actual $k$-means
algorithm. However, these papers are usually concerned with the {\em value}
of the $k$-means objective function at the final solution, not with
the {\em position} of the final centers. As far as we know, our paper is the
first one which analyzes the ``regions of attractions'' of the
different local optima of the actual $k$-means algorithm and derives results on the stability of the
$k$-means clustering itself.

\section{Notation and assumptions} \label{sec-preliminaries}

In the following we assume that we are given a set of $n$ data points
$X_1, ..., X_n \in \R$ which have been drawn i.i.d. according to
some underlying distribution $\P$.
For a center vector $c = (c_1, ...,c_{K'})$ with $c_i \in \R^d$ we denote the cluster induced by center
$c_k$ with $\Ccal_k(c)$. The number of points in this cluster is denoted $N_k(c)$.
The clustering algorithm we study in this paper is the standard
$k$-means algorithm. We denote the  initial centers by $c_1^{<0>},
...,c_{K'}^{<0>}$ with $c_i^{<0>} \in \R$, and the centers after step $t$ of the algorithm as
$c_1^{<t>}, ...,c_{K'}^{<t>}$. By $\ktrue$ we denote the true number of clusters,
by $\kalg$ the number of clusters constructed by the $k$-means
algorithm.
It attempts to minimize the $k$-means
objective function
\ba
W_n: \R^{d K'} \to \R, W_n(c_1, ..., c_{K'})
= \frac{1}{2} \sum_{i=1}^n
\min_{k =  1,..,{K'}}
|| c_k - X_i||^2.
\ea
We now restate the $k$-means algorithm: \\
{\tt
Input: $X_1,..., X_n \in \R^d$, $K' \in \Nat$\\
 Initialize the centers $c_1^{<0>}, ...,c_{K'}^{<0>} \in \R^d$\\
 Repeat until convergence: \\
 1. Assign data points to closest centers. \\
 2. Re-adjust cluster means:
 \begin{equation} \label{kmeans-step2}
 c^{<t+1>}_k  =  \frac{1}{N_k(c^{<t>})} \sum_{i: \; X_i \in \Ccal_k(c^{<t>})_k} X_i
 \end{equation}
 Output: $c = (c_1^{<final>}, ..., c_{K'}^{<final>})$.\\
 }

Traditionally, the instability of a clustering algorithm is defined as
the mean (with respect to the random sampling of data points) minimal matching distance between two clusterings obtained on two different set of data points. For the actual $k$-means algorithm, a second
random process is the random initialization (which has not been taken
into account in previous literature). Here we additionally have to
take the expectation over the random initialization when computing the
stability of an algorithm. In this paper we
will derive qualitative rather than quantitative results on
stability, thus we omit more detailed formulas. \\

In the following we restrict our
attention to the simple setting where the underlying distribution is
a mixture of Gaussians on $\R$ and we have access to
an infinite amount of data from $\P$.
 In particular, instead of estimating means empirically when calculating the new centers of a $k$-means
 step we assume access to the true means. In this case, the update step
 of the $k$-means algorithm can be written as
$$c^{<t+1>}_k  = \frac{\int_{\Ccal_k(c^{<t>})} x f(x) dx}{\int_{\Ccal_k(c^{<t>})} f(x) dx}$$
where $f$ is the density of the probability distribution $\P$.
Results in the finite data case can be derived by the help of
concentrations inequalities. However, as this introduces heavy
notation and our focus lies on the random initialization rather than the
random drawing of data points we skip the
details. To further set up notation we denote $\varphi_{\mu,\sigma}$
the pdf of a Gaussian distribution with mean $\mu$ and variance
$\sigma$.  We also denote $f(x) = \sum_{k=1}^K w_k
\varphi_{\mu_k,\sigma}$ where $K$ is the number of Gaussians, the
weights $w_k$ are positive and sum to one, the means $\mu_{1:K} =
(\mu_1, \hdots, \mu_K)$ are ordered, $\mu_1 \leq \hdots \leq
\mu_K$. The minimum separation between
two Gaussians is denoted by $\Delta=\min_k(\mu_{k+1}-\mu_k)$. For the
standard normal distribution we denote the pdf as $\phi$ and the cdf
as $\Phi$.  \\

\section{The level sets approach} \label{sec-levelsets}

In this section we want to prove that
if we run the $k$-means algorithm with $K'=2$ and
$K'=3$ on a mixture of two Gaussians, then the resulting clustering depends exclusively on the initial configuration.
More precisely if we initialize the algorithm such that each cluster
gets at least one center and the initial centers are ``close enough''
to the true cluster means, then during the course of the algorithm the
initial centers do not leave the cluster they had been placed in. This
implies stability for $K'=2$ since there is only one possible configuration satisfying this constraint.
On the other hand for $K'=3$ we have two possible configurations, and thus instability will occur. \\

The following function plays an important role in our analysis:
\ba
H: \R^2 \to \R, \; H(x,y) = x \Phi(-x+y) - \varphi(-x+y) .
\ea
Straightforward computations show that
for any $\mu, \sigma, \alpha$ and $h$ one has
\banum \label{lemma:H}
\int_{-\infty}^h (x - \mu + \alpha) \varphi_{\mu,\sigma} (x) dx =
\sigma H\left(\frac{\alpha}{\sigma},
  \frac{h+\alpha-\mu}{\sigma}\right) .
\eanum

We describe necessary and sufficient conditions to obtain stability results for particular ``regions'' in terms of the level
sets of $H$.

\subsection{Stability in the case of  two initial centers} \label{sec:K2}
We consider the square
$S_a = [\mu_1 - a, \mu_1 + a] \times [\mu_2 - a, \mu_2 + a]$ in
$\R^2$.
The region $S_a$ is called {\em a stable region} if
\begin{equation} \label{eq:13}
c^{<0>} \in S_a \Rightarrow c^{<1>} \in S_a
\end{equation}

\begin{proposition}[Stable region for $K'=2$] \label{th:stabregion2}
Equation (\ref{eq:13}) is true if and only if the following four inequalities are satisfied:
\begin{footnotesize}
\begin{align}
&\bullet w_1 H\left(\frac{a}{\sigma}, \frac{\Delta}{2 \sigma}\right) + w_2 H\left(\frac{a+\Delta}{\sigma}, \frac{\Delta}{2 \sigma}\right) \geq 0 \label{eq:8}\\ \nonumber \\
&\bullet w_1 H\left(-\frac{a}{\sigma}, \frac{\Delta}{2 \sigma}\right) + w_2 H\left(\frac{-a+\Delta}{\sigma}, \frac{\Delta}{2 \sigma}\right) \leq 0 \label{eq:9}\\ \nonumber \\
&\bullet w_1 H\left(\frac{a - \Delta}{\sigma}, -\frac{\Delta}{2 \sigma}\right) + w_2 H\left(\frac{a}{\sigma}, \frac{\Delta}{2 \sigma}\right) \geq 0 \label{eq:10}\\ \nonumber \\
&\bullet w_1 H\left(\frac{-a-\Delta}{\sigma}, -\frac{\Delta}{2 \sigma}\right) + w_2 H\left(-\frac{a}{\sigma}, -\frac{\Delta}{2 \sigma}\right) \leq 0 \label{eq:11}
\end{align}
\end{footnotesize}
\end{proposition}

\begin{proof}
Similar to the proof of Proposition \ref{th:stabregion}, see below.
\end{proof}

This proposition gives necessary and sufficient conditions for the
stability of $k$-means in the case $K'=2$.  In the
following corollary we show an example of the kind of result we can
derive from Proposition \ref{th:stabregion2}.
Note that the
parameters $a$ and $\Delta$ only appear relative to $\sigma$. This
allows us to consider an arbitrary $\sigma$.

\begin{corollary}[Stability for $K'=2$] \label{coro:num2}
Assume that $\min(w_1,w_2)=0.2$ and $\Delta = 7 \sigma$. Assume that we have an initialization scheme satisfying:
 \begin{itemize}
 \item with probability at least $1-\delta$ we have one initial center
   within $2.5 \sigma$ of $\mu_1$ and one within $2.5 \sigma$ of
   $\mu_2$.
 \end{itemize}
 Then $k$-means is stable in the sense that with probability at least
 $1-\delta$ it converges to a solution with one center within $2.5
 \sigma$ of $\mu_1$ and one within $2.5 \sigma$ of $\mu_2$.
\end{corollary}

\begin{proof}
We simply check numerically that for $a= 2.5 \sigma, \Delta = 7 \sigma$ and $w_1=0.2$ (we also check $w_2=0.2$)
Equations (\ref{eq:8}) - (\ref{eq:11}) are true. Then by Proposition \ref{th:stabregion2} we know that $S_a$ is a stable region which implies the result.
\end{proof}

\subsection{Instability in the case of 3 centers} \label{sec:K3} The
case of $3$ centers gets more intricate. %
Consider the prism $ T_{a,b,\epsilon}$ and its symmetric version
$sym(T_{a,b,\epsilon})$ in $\R^3$:
\ba
 T_{a,b,\epsilon} = &\{c \in \R^3 : c_1 \leq c_2 \leq c_3, \\
& c \in [\mu_1-a, \mu_1+a-\epsilon]\times [\mu_1-a+\epsilon,\mu_1+a] \times [\mu_2-b,\mu_2+b] \}  \\
 sym(T_{a,b,\epsilon}) = &  \{c \in \R^3 : c_1 \leq c_2 \leq c_3, \\
& c \in [\mu_1-b, \mu_1+b]\times [\mu_2-a,\mu_2+a-\epsilon] \times [\mu_2-a+\epsilon,\mu_2+a] \}  .
\ea
If we have an initialization scheme such that each cluster
gets at least one center and the initial centers are close enough to
the true cluster means, then we initialize either in $T_{a,b,\epsilon}$
or $sym(T_{a,b,\epsilon})$. Thus, if these regions are stable in the
following sense:
\begin{equation}
c^{<0>} \in T_{a,b,\epsilon} \Rightarrow c^{<1>} \in T_{a,b,\epsilon} \label{eq:1}
\end{equation}
then the global k-means algorithm will be instable, leading either to
a clustering in $T_{a,b,\epsilon}$ or
$sym(T_{a,b,\epsilon})$. Expressed in the
terms used in the introduction, the algorithm will be initialized
with different configurations and thus be instable.

\begin{proposition}[Stable region for $K'=3$] \label{th:stabregion}
Equation (\ref{eq:1}) is true if and only if all the following inequalities are satisfied:
\begin{footnotesize}
\begin{align}
\bullet & w_1 H\left(\frac{a}{\sigma}, \frac{\epsilon}{2 \sigma}\right) + w_2 H\left(\frac{a+\Delta}{\sigma}, \frac{\epsilon}{2 \sigma}\right) \geq 0 \label{eq:2}\\ \nonumber \\
\bullet & w_1 H\left(\frac{-a+\epsilon}{\sigma}, \frac{\epsilon}{2 \sigma}\right) + w_2 H\left(\frac{-a+\Delta+\epsilon}{\sigma}, \frac{\epsilon}{2 \sigma}\right) \leq 0 \label{eq:3} \\ \nonumber \\
\nonumber  \bullet & w_1 H\left(\frac{a-\epsilon}{\sigma}, \frac{a-b+\Delta-\epsilon}{2 \sigma}\right) + w_2 H\left(\frac{a-\epsilon+\Delta}{\sigma}, \frac{a-b+\Delta-\epsilon}{2 \sigma}\right) \\
& \geq w_1 H\left(\frac{a-\epsilon}{\sigma}, -\frac{\epsilon}{2 \sigma}\right) + w_2 H\left(\frac{a-\epsilon+\Delta}{\sigma}, -\frac{\epsilon}{2 \sigma}\right) \label{eq:4} \\ \nonumber \\
\nonumber \bullet & w_1 H\left(- \frac{a}{\sigma}, \frac{b-a+\Delta}{2 \sigma}\right) + w_2 H\left(\frac{-a+\Delta}{\sigma}, \frac{b-a+\Delta}{2 \sigma}\right) \\
& \leq w_1 H\left(- \frac{a}{\sigma}, -\frac{\epsilon}{2 \sigma}\right) + w_2 H\left(\frac{-a+\Delta}{\sigma}, -\frac{\epsilon}{2 \sigma}\right) \label{eq:5} \\ \nonumber \\
\nonumber \bullet & w_1 H\left(\frac{b-\Delta}{\sigma}, \frac{b-a-\Delta+\epsilon}{2 \sigma}\right) + w_2 H\left(\frac{b-\Delta}{\sigma}, \frac{b-a-\Delta+\epsilon}{2 \sigma}\right) \\
& \leq b/\sigma - w_1 \Delta/\sigma \label{eq:6} \\ \nonumber\\
\nonumber \bullet & w_1 H\left(\frac{-b-\Delta}{\sigma}, \frac{a-b-\Delta}{2 \sigma}\right) + w_2 H\left(-\frac{b}{\sigma}, \frac{a-b-\Delta}{2 \sigma}\right) \\
& \geq - b / \sigma - w_1 \Delta / \sigma \label{eq:7}
\end{align}
\end{footnotesize}
\end{proposition}

\begin{proof}{\em (Sketch)}
  Let $c^{<0>} \in T_{a,b,\epsilon}$. Note that the
  $k$-means algorithm in one dimension does not change the orders of
  centers, hence %
$c^{<1>}_1 \leq c^{<1>}_2 \leq
c^{<3>}_1$. By the definition of $T_{a,b,\epsilon}$, to prove
  that after the first step of $k$-means the centers $c^{<1>}$ are still
  in $T_{a,b,\epsilon}$ we have to check six constraints.  Due to
  space constraints, we only show how to prove that the first constraint $c^{<1>}_1 \geq
  \mu_1 - a$ is equivalent to Equation (\ref{eq:2}).
  The other conditions can be treated similarly.
  \newline

The update step of the $k$-means algorithm on the underlying
distribution readjusts the centers to the actual cluster
means:
$$c^{<1>}_1 = \frac{1}{\int_{-\infty}^{\frac{c^{<0>}_1 + c^{<0>}_2}{2}} f(x)} \int_{-\infty}^{\frac{c^{<0>}_1 + c^{<0>}_2}{2}} x f(x).$$
Thus, $c^{<1>}_1 \geq \mu_1 - a$ is equivalent to
$$\int_{-\infty}^{\frac{c^{<0>}_1 + c^{<0>}_2}{2}} (x - \mu_1 + a) f(x) \geq 0.$$
Moreover, the function $h \mapsto \int_{-\infty}^{h} (x - \mu_1 + a) f(x)$ is nondecreasing for $h \in [\mu_1 - a, +\infty)$.
Since $c^{<0>} \in T_{a,b,\epsilon}$ we know that $(c^{<0>}_1 + c^{<0>}_2)/2 \geq \mu_1 - a + \epsilon /2$ and thus the statement $\forall c^{<0>} \in T_{a,b,\epsilon}, c_1^1 \geq \mu_1 - a$ is equivalent to
$$\int_{-\infty}^{\mu_1 - a + \epsilon /2} (x - \mu_1 + a) f(x) \geq 0 .$$
We can now apply Eq.~\eqref{lemma:H} with the following decomposition
to get Eq.~(\ref{eq:2}):
\small{
\ba
& \int_{-\infty}^{\mu_1 - a + \epsilon /2} (x - \mu_1 + a) f(x) \\
& = w_1 \int_{-\infty}^{\mu_1 - a + \epsilon /2} (x - \mu_1 + a) \varphi_{\mu_1, \sigma}
+ w_2 \int_{-\infty}^{\mu_1 - a + \epsilon /2} (x - \mu_2 + \Delta + a) \varphi_{\mu_2, \sigma} .
\ea}

\end{proof}

A simple symmetry argument allows us to treat the stability of the symmetric prism.
\begin{proposition} \label{prop:sym} If $T_{a,b,\epsilon}$ is stable
  for the pdf $f(x) = w_1 \varphi_{\mu_1,\sigma} + w_2
  \varphi_{\mu_2,\sigma}$ and $\tilde{f}(x) = w_2
  \varphi_{\mu_1,\sigma} + w_1 \varphi_{\mu_2,\sigma}$,  then the same
  holds for $sym(T_{a,b,\epsilon})$.
\end{proposition}

\begin{proof}
The $k$-means algorithm is invariant with respect to translation of the real axis as well as to changes in its orientation. Hence if $T_{a,b,\epsilon}$ is stable under $f$ (resp. $\tilde{f}$), so is $sym(T_{a,b,\epsilon})$ under $\tilde{f}(x) = w_2 \varphi_{\mu_1,\sigma} + w_1 \varphi_{\mu_2,\sigma}$ (resp. $f$).
\end{proof}

\begin{corollary}[Instability for $K'=3$] \label{coro:num3}
Assume that $\min(w_1,w_2)=0.2$ and $\Delta = 14.5 \sigma$. Assume that we have an initialization scheme satisfying:
 \begin{itemize}
   \item with probability at least $(1-\delta)/2$ we have 2 initial centers within $2.5 \sigma$ of $\mu_1$ and 1 initial center within $2.5 \sigma$ of $\mu_2$
   \item with probability at least $(1-\delta)/2$ we have 1 initial centers within $2.5 \sigma$ of $\mu_1$ and 2 initial centers within $2.5 \sigma$ of $\mu_2$
 \end{itemize}
Then $k$-means is instable: with probability $(1-\delta)/2$ it will converge to a solution with two centers within $3.5 \sigma$ of $\mu_1$ and with probability $(1-\delta)/2$ to a solution with two centers within $3.5 \sigma$ of $\mu_2$.
\end{corollary}

\begin{proof}
We simply check numerically that for $a= 3.5 \sigma$, $b = 2.5
  \sigma$, $\epsilon= \sigma$, $\Delta = 14.5 \sigma$ and $w_1=0.2$ (we also check $w_2=0.2$)
Equations (\ref{eq:2}) - (\ref{eq:7}) are true. Then by Proposition \ref{th:stabregion} and Proposition \ref{prop:sym} we know that $T_{3.5 \sigma, 2.5 \sigma, \sigma}$ and its symmetric $sym(T_{3.5 \sigma, 2.5 \sigma, \sigma})$ are stable regions which implies the result.
\end{proof}

\section{Towards more general results: the geometry of the solution
  space of $k$-means}

In the section above we proved by a level set approach that in a very simple setting, 
if we initialize the $k$-means algorithm ``close enough'' to the true cluster centers, then the
initial centers do not move between clusters. However we would like to obtain this result in a more
general setting. We believe that to achieve this goal in a
systematic way one has to understand the structure of the solution
space of $k$-means. We identify the solution space with the space
$\R^{dK'}$ by representing a set of $K'$ centers $c_1, ..., c_{K'} \in \R^d$ as
a point $c$ in the space $\R^{d K'}$.  Our goal in this section is to
understand the ``shape'' of the $k$-means objective function on this
space. Secondly, we want to understand how the $k$-means algorithm
operates on this space. That is, what can we say about the
``trajectory'' of the $k$-means algorithm from the initial point to
the final solution? For simplicity, we state some of the results in
this section only for the case where the data space is one
dimensional. They also hold in $\R^d$, but are more nasty to write up.  \\

First of all, we want to compute the derivatives of $W_n$ with respect
to the individual centers.

\begin{proposition}[Derivatives of $k$-means] \label{prop-derivatives}
Given a finite data set $X_1, ..., X_n \in \R$. For $k, l \in \{1, ...,
K'\}$ and $i \in \{1, ..., n\}$ consider the hyperplane in $\R^{K'}$ which
is defined by
\ba
H_{k, l, i} := \{ c \in \R^{K'} : X_i = (c_k + c_l) / 2 \}.
\ea
Define the set $H := \cup_{k, l = 1}^{K'} \cup_{i=1}^n H_{k,l,i}$.
Then we have:
\begin{enumerate}
\item
$W_n$ is differentiable on $\R^{K'} \setminus H$ with partial derivatives
\ba
\frac{\partial W_n(c) }{\partial c_k}
= \sum_{i: \; X_i \in \Ccal_k} (c_k - X_i).
\ea

\item The second partial derivatives of $W_n$ on $\R^{K'} \setminus H$
  are
\ba
\frac{\partial W_n(c) }{\partial c_k \partial c_l } = 0
&& \text{ and } &&  \frac{\partial W_n(c) }{\partial c_k \partial c_k } = N_k.
\ea
\item The third derivatives of $W_n$ on $\R^{K'} \setminus H$ all vanish.
\end{enumerate}
\end{proposition}

\begin{proof}
First of all, note that the sets $H_{k,l,i}$ contain the center
vectors for which there exists a data point $X_i$ which lies on the
boundary of two centers $c_k$ and $c_l$. Now let us look at the first
derivative. We compute it by foot:
\ba
\frac{\partial W_n(c)}{\partial c_{k}} =
&
\lim_{h \to 0} \frac{1}{h} (W_n(c_1, ..., c_K) - W_n(c_1, ..., c_{k} + h, ..., c_K))
\ea
When $c \not\in H$ we know that no data point lies on the boundary
between two cluster centers. Thus, if $h$ is small enough, the
assignment of data points to cluster centers does not change if we
replace $c_k$ by $c_k +h$. With this property, the expression above is
trivial to compute and yields the first derivative, the other derivatives follow similarly.
\end{proof}

A straightforward consequence  is as follows:
\begin{proposition}[$k$-means does Newton iterations] \label{prop:newton}
The update steps performed by the $k$-means algorithms are exactly the same as
update steps by a Newton optimization.
\end{proposition}
\begin{proof}
This proposition follows directly from Proposition
\ref{prop-derivatives}, the definition of the Newton iteration on
$W_n$ and the definition of the $k$-means update step. This fact has
also been stated (less rigorously and without proof) in
\citet{BotBen95}.
\end{proof}

Together, the two propositions show an interesting picture. We have
seen in Proposition \ref{prop-derivatives} that the $k$-means
objective function $W_n$ is differentiable on $\R^{K'} \setminus H$. This
means that the space $\R^{K'}$ is separated into many cells with
hyperplane boundaries $H_{k,l, i}$.  By construction, the cells are
convex (as they are intersections of half-spaces). 
Our finding  means that each data set $X_1, ..., X_n$
induces a partitioning of this solution space into convex
cells. 
To avoid confusion,
at this point we would like to stress again that we are not looking at
a fixed clustering solution on the data space (which can be described
by cells with hyperplane boundaries, too), but at the space of all center vectors
$c$. It is easy to see that all centers
$c$ within one cell correspond to exactly one clustering of the data
points. As it is well known that the $k$-means algorithm never visits
a clustering twice, we can conclude that each cell is visited at most
once by the algorithm.Within each cell, $W_n$ is quadratic (as the
third derivatives vanish). Moreover, we know that $k$-means behaves as
the Newton iteration. On a quadratic function, the Newton optimization
jumps in one step to the minimum of the function. This means that if
$k$-means enters a
cell that contains a local optimum of the $k$-means objective function,
then the next step of $k$-means jumps to this local optimum
and stops. \\

Now let us look more closely at the trajectories of the $k$-means
algorithm. The paper by \citet{ZhaDaiTun08} inspired us to derive the following property.

\begin{proposition}[Trajectories of $k$-means] \label{prop:traj}
Let $c^{<t>}$ and $c^{<t+1>}$ be two consecutive solutions visited by
the $k$-means algorithm. Consider the line connecting those two
solutions in $\R^{K'}$, and let
$c^{\alpha} = (1-\alpha) c^{<t>} + \alpha c^{<t+1>}$ be a point on
this line (for some $\alpha
\in [0,1]$). Then
$W_n(c^{\alpha} ) \leq W_n(c^{<t>})$.
\end{proposition}

\begin{proof}
The following inequalities hold true:
\ba
W_n(c^{\alpha}) & = \frac{1}{2} \sum_{k=1}^K \sum_{i \in \Ccal_k(c^{\alpha})} ||X_i - c_k^{\alpha}||^2 \\
 & \leq \frac{1}{2} \sum_{k=1}^K \sum_{i \in \Ccal_k(c^{t})} ||X_i - c_k^{\alpha}||^2 \\
 & \leq \frac{1}{2} \sum_{k=1}^K \sum_{i \in \Ccal_k(c^{t})} \alpha ||X_i - c_k^{t}||^2 + (1-\alpha) ||X_i - c^{t+1}_k||^2 \\
 & \leq \alpha W_n(c^t) + (1-\alpha) W_n(c^{t+1})
 \ea
 For the first and third inequality we used the fact that
 assigning points in $\Ccal_k(c)$ to the center $c_k$ is the best thing
 to do to minimize $W_n$.
 For the second inequality we used that $x
 \rightarrow ||x||^2$ is convex. The proof is concluded by noting that
 $W_n(c^{<t>}) \leq W_n(c^{<t+1>})$.
\end{proof}

We believe that the properties of the $k$-means objective function and
the algorithm are the key to prove more general stability
results. However, there is still an important piece
missing, as we are going to explain now.
Since $k$-means performs Newton iterations on $W_n$, one could expect
to get information on the trajectories in the configuration space by
using a Taylor expansion of $W_n$. However, as we have seen above,
each step of the $k$-means algorithm crosses one of the hyperplanes
$H_{k,l,i}$ on which $W_n$ is non-differentiable. Hence, a direct Taylor
expansion approach on $W_n$ cannot work. On the other hand,
surprisingly one can prove that
the limit objective function $W := \lim \frac{1}{n} W_n$ is almost surely a
continuously differentiable function on $\R^{K'}$ (we omit the proof in
this paper). Thus one may hope that one could first study the
behavior of the algorithm for $W$, and then apply concentration
inequalities to carry over the results to $W_n$.  Unfortunately, here
we face another problem: one can prove that in the limit case, a step
of the $k$-means algorithm  is {\em not} a Newton iteration on $W$.
\newline

Proposition \ref{prop:traj} directly evokes a scheme to design stable
regions. Assume that we can find two regions $A \subset B \subset \R^{K'}$ of full
rank and such that
\begin{equation} \label{eq:poly}
\max_{x \in \partial A} W_n(x) \leq \min_{x \in \partial B} W_n(x).
\end{equation}
Then, if we initialize in $A$ we know that we will converge to a
configuration in $B$. This approach sounds very promising. However, we
found that it was impossible to satisfy both Equation (\ref{eq:poly})
and the constraint that $A$ has to be "big enough" so that we
initialize in $A$ with high probability. \\ %

Finally, we would like to elaborate on a few more complications
towards more general results:

\blobb{On a high level, we want to prove that if $K'$ is slightly
  larger than the true $K$, then $k$-means is instable. On the other
  hand, if $K'$ gets close to the number $n$ of data points, we
  trivially have stability again. Hence, there is some kind of
  ``turning point'' where the algorithm is most instable. It will be
  quite a challenge to work out how to determine this turning point. }

\blobb{Moreover, even if we have so many data points that the above
  problem is unlikely to occur, our analysis breaks down if $K'$ gets
  too large. The reason is that if $K'$ is much bigger than $K$, then
  we cannot guarantee any more that initial centers will be in stable
  regions. Just the opposite will happen: at some point we will have
  outliers as initial centers, and then the behavior of the algorithm
  becomes rather unpredictable. }

\blobb{Finally, consider the case of $K' < K$. As we have already
  mentioned in the introduction, in this case it is not necessarily
  the case that different initial configurations lead to different
  clusterings. Hence, a general statement on (in)stability is not
  possible in this case. This also means that the tempting conjecture
  ``the true $K$ has minimal stability'' is not necessarily true. }

\section{An initialization algorithm and its analysis} \label{sec-init}

\begin{figure}[t]
\begin{small}
\begin{boxit}
Algorithm \initalg
\begin{enumerate}
\item[]\hspace{\backitem}Input: $w_{min}$,  number of centers $K'$
\item \label{step:pick0}Initialize with $L$ random points
  $\cenzero_{1:L}$, $L$ computed by (\ref{eq:Lp0})
\item Run one step of $k$-means, that is
\begin{enumerate}
\item \label{step:assign0} To each center $\cenzero_j$ assign region $\Ccal^0_j, \,j=1:L$
\item \label{step:centers1}Calculate $\cenunu_{1:L}$ as  the centers of
  mass of regions $\Ccal^0_{1:L}$
\end{enumerate}
\item \label{step:prune}Remove all centers $\cenunu_j$ for which $P[\Ccal^1_j]\leq p_0$, where $p_0$ is given by (\ref{eq:Lp0}). We are left with $\cenunu_{j'},\,j'=1:L'$.
\item \label{step:HS}Choose $K'$ of the remaining centers by the {\sc MinDiam} heuristic \benum
\item \label{step:inHS} Select one center at random.
\item Repeat until $K'$ centroids are selected:
\item[] Select the centroid $\cenunu_q$ that maximizes the minimum distance to the already selected centroids.
\eenum
\item[]\hspace{\backitem} Output: the $K'$ selected centroids $\cenunu_{k},\,k=1:K'$
\end{enumerate}
\end{boxit}
\end{small}
\caption{ \label{fig:init-alg}
The  \initalg~ initialization}
\end{figure}

We have seen that one can prove results on clustering stability for
$k$-means if we use a "good" initialization scheme which tends to
place initial centers in different Gaussians.
We now show that an established initialization algorithm, the
\initalg~ initialization described in Figure \ref{fig:init-alg} has
this property, i.e it has the effect of placing the initial centroids
in disjoint, bounded neighborhoods of the means $\mu_{1:K}$. This often
rediscovered algorithm is credited to \citet{hochbaum:85}. In
\citet{dasgupta:07} it was analyzed it in the context of the EM
algorithm. Later \citet{SreShaRow06} used it in experimental
evaluations of EM, and it was found to have a significant advantage
w.r.t more naive initialization methods in some cases.  While this and
other initializations have been extensively studied in conjunction
with EM, we are not aware of any studies of \initalg~ for $k$-means.

We make three necessary conceptual assumptions. Firstly to ensure that $K$ is
well-defined we assume that the mixture weights are bounded below by a known
weight $w_{min}$.
\begin{assumption}\label{ass:wmin} $w_k\geq w_{min}$ for all k.
\end{assumption}
We also require to know a lower bound $\Delta$ and an upper bound
$\Delta_{max}$ on the separation between two Gaussians, and we
 assume that these separations are
``sufficiently large''.
In addition, later we shall make several technical assumptions
related to a parameter $\tau$ used in the proofs, which also amount
to conditions on the separation.  These assumptions shall be made precise later.

\begin{theorem}[\initalg~ Initialization]\label{thm:initialization0} Let
$f=\sum_1^Kw_k\phi_{\mu_k,1}$ be a mixture of $K$ Gaussians with
centers $\mu_{1:K},\,\mu_k\leq\mu_{k+1}$, and unit variance. Let
$\tau\in(0,0.5),\,\delta_{miss}>0,\,\delta_{impure}$ defined in
Proposition \ref{thm:prob-impure}. If we run
Algorithm \initalg~ with any $2 \leq K' \leq {1}/{w_{min}}$, then,
subject to Assumptions 1, 2, 3, 4, 5 (specified later), with probability
$1-2\delta_{miss}-\delta_{impure}$  over the
initialization there exist $K$ disjoint intervals $\tilde{A}_k$,
specified in Section \ref{sec:step4},
one for each true mean $\mu_k$, so that all $K'$ centers
$\cenunu_{k'}$ are contained in $\bigcup_k \tila_k$ and
\begin{eqnarray}
\mbox{if } K'=K, \mbox{ each } \tila_k \mbox{ will contain exactly one center } \cenunu_{k'}, \\
\mbox{if } K'<K, \mbox{ each } \tila_k \mbox{ will contain at most one center } \cenunu_{k'}, \\
\mbox{if } K'>K, \mbox{ each } \tila_k \mbox{ will contain at least one center } \cenunu_{k'}.
\end{eqnarray}
\end{theorem}

The idea to prove this result is to show that the following statements
hold with high probability.  By selecting $L$ preliminary
centers in step \ref{step:pick0} of \initalg, each of the Gaussians
obtains at least one center (Section \ref{sec:pick0}). After steps
\ref{step:assign0}, \ref{step:centers1} we obtain ``large'' clusters (mass  $>p_0$) and ``small'' ones (mass $\leq p_0$). A cluster can also be
``pure'' (respectively ``impure'') if most of its mass comes from a single Gaussian
(respectively from several Gaussians). Step \ref{step:prune} removes all ``small'' cluster centers,
but (and this is a crucial step of our argument) w.h.p it will also
remove all ``impure'' cluster centers (Section \ref{sec:impure}). The
remaining clusters are ``pure'' and ``large''; we show (Section
\ref{sec:almost-pure}) that each of their centers is reasonably close
to some Gaussian mean $\mu_k$. Hence, if the Gaussians are well
separated, the selection of final centers $\cenunu_q$ in step
\ref{step:HS} ``cycles through different Gaussians'' before visiting a
particular Gaussian for the second time (Section \ref{sec:step4}). The rest of this section
outlines these steps in more details.

\subsection{Step \ref{step:pick0} of \initalg. Picking the initial centroids $\cenzero$}
\label{sec:pick0}

We need to pick a number of initial centers $L$ large enough that each
Gaussian  has at least 1 center w.h.p. We formalize this here and find
a value for $L$ that ensures the probability of this event is at least
$1-\delta_{miss}$, where $\delta_{miss}$ is a tolerance of our
choice.
Another event that must be avoided for a ``good'' initialization is
that all centroids $\cenzero_j$ belonging to a Gaussian end up with
initial clusters $\clust^0_j$ that have probability less than $p_0$. If this
happens, then after thresholding, the respective Gaussian is left with no
representative centroid, i.e it is ``missed''. We set the tolerance
for this event to $\delta_{thresh}=\delta_{miss}$.
Let $t=2\Phi(-\Delta/2)$ the {\em tail probability} of a cluster and $A_k$
the symmetric neighborhood of $\mu_k$ that has
$\phi_{\mu_k,1}(A_k)=1-t$.
\begin{proposition}\label{lem:pick0} If we choose
\banum \label{eq:Lp0}
L\;\geq\; \left. \left. \left. \left(\ln\frac{1}{\delta_{miss}w_{min}}\right) \right/ \right((1-t)w_{min} \right)
\;\; \text{ and } \;\;
p_0\;=\;\frac{1}{eL}
\eanum
then the probability
over all random samplings of centroids $\cenzero_{1:L}$ that at least
one centroid $c_j^{<0>}$ with assigned mass  $P[\clust^0_j]\geq p_0$ can be found in
each $A_k,\,k=1:K$, is greater or equal to $1-2\delta_{miss}$.
\end{proposition}

The proof of this result is complicated but standard fare (e.g. Chernoff
bounds) and is therefore omitted. %

After steps \ref{step:pick0}, \ref{step:assign0} and
\ref{step:centers1} of \initalg~
are performed, we obtain centers $\cenunu_{1:L}$
situated at the centers of mass of their respective clusters
$\clust^1_{1:L}$. Removing the centers of small clusters follows. We
now describe a  beneficial effect of this step.

\subsection{Step \ref{step:prune} of \initalg. Thresholding removes impure clusters}
\label{sec:impure}

We introduce the concept of {\em purity} of a cluster, which is
related to the ratio of points from a certain Gaussian w.r.t to the
total probability mass of the cluster. Denote $P_k$ the probability
distribution induced by the $k$-th Gaussian $\phi_{\mu_k,1}$.
\begin{definition} A cluster $\clust$ is {\em $(1-\tau)$-pure} if most of its
points come from a single Gaussian, i.e if $w_kP_k[\clust]\geq
(1-\tau)P[\clust]$, with $\tau<1/2$ being a  positive constant. A
cluster which is not $(1-\tau)$-pure is {\em $\tau$-impure} (or simply
{\em impure}).
\end{definition}
The values of $\tau$ that we consider useful are of the order
$0.001-0.02$ and, as it will appear shortly, $\tau<w_{min}/2$.
\commentout{
\begin{definition} The {\em local purity} of Gaussian $k$ relative to
  Gaussian $k'$ at point $x$ is
\ba
\gamma_{k,k'}\;=\;\frac{w_{k'}f_{k'}(x)}{w_{k}f_{k}(x)}
\ea
and {\em local purity} of Gaussian $k$ at $x$ is
\ba
\gamma_k(x)\;=\;\sum_{k'\neq  k}\gamma_{k,k'}
\;=\;\frac{1-w_kf_k(x)}{w_{k}f_{k}(x)}
\ea
\end{definition}
Note that the definition is slightly counterintuitive: a Gaussian
becomes more pure, $\gamma_{k}$ decreases.
It is easy to see that the local purities $\gamma_{k,k'}$ are convex and  log-linear functions of $x$. For instance,
\ba
\gamma_{21}(x)\;=\;\frac{w_1}{w_2}e^{-(x-\mu_1)^2/2+(x2-\mu_2)^2/2}
\;=\;e^{x(\mu_1-\mu_2)+\frac{\mu_2^2-\mu_1^2}{2}+\ln \frac{w_1}{w_2}}
\ea
}%
The purity of a cluster helps in the following way: if a cluster
is pure, then it can be ``tied'' to one of the Gaussians. Moreover,
its properties (like center of mass) will be dictated by the Gaussian
to which it is tied, with the other Gaussians' influence being
limited; Section \ref{sec:almost-pure} exploits this idea.

But there will also be clusters that are impure, and so they cannot be
tied to any Gaussian. Their properties will be harder to analyze, and
one expects their behavior to be less predictable. Luckily, impure
clusters are very likely small. As we show now, the chance of having
an impure cluster with mass larger than $p_0$ is bounded by a
$\delta_{impure}$ which we are willing to tolerate.

Because of limited space, we leave out the long and complex rigourous
proofs of this result, and give here just the main ideas. %
Let $\clust_j=[z_1,z_2]$ be a $\tau$-impure cluster, with
$P[\clust_j]\geq p_0$, $c_j$ the centroid that generates $\clust_j$
(not necessarily at its center of mass) and $c_{j-1},c_{j+1}$ the
centroids of the adjacent clusters (not necessarily centers of
mass). As one can show,
even though an impure cluster contains some
probability mass from each Gaussian, in most of this section we only
need consider the two Gaussians which are direct neighbors of
$\clust$. Let us denote the parameters of these (consecutive)
Gaussians by $\mu_{1,2}, w_{1,2}$.

For the purpose of the proof, we are looking here at the situation after step
\ref{step:assign0}, thus the centroids $c_{j-1,j,j+1}$ should  be
$\cenzero_{j-1,j,j+1}$, but we renounce this convention temporarily to
keep the notation light. We want to bound the probability of cluster
$\clust_j$ being impure and large. Note that Step \ref{step:centers1}
of the \initalg~ does not affect either of these properties, as it
only acts on the centers.

\begin{SCfigure}[40][t]
\centering
\includegraphics[width=0.5\textwidth]{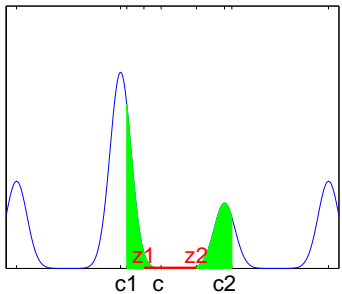}
\caption{\small \label{fig:impure}
Concrete example of a large impure cluster $[z1,z2]$;
$c1,c,c2$ represent the consecutive cluster centers
$\cenzero_{j-1},\cenzero_j,\cenzero_{j+1}$. We demonstrate that if
$P[z1,z2]>p_0$ then the interval $[c1,c2]$ (which is twice its length)
must have mass $>p_1>>p_0$. If $L$ is large enough, having such a
large interval contain a single $c_j$ is improbable. Numerical values:
mixture with $\Delta=10,w_{min}=0.15$, impurity
$\tau([z_1,z_2])=0.07$, $P[z1,z2]=0.097$, $P[c1,c2]=0.24$; using
$\delta_{miss}=0.02,\tau=0.015$ one gets
$L=38,\,p_0=0.095<P[z1,z2],\,p_1=0.0105<P[c1,c2],\,\delta_{impure}=0.016>>(1-P[c1,c2])^{L-1}=0.00003$}
\end{SCfigure}

A simple observation is the following. Since
$z_1=\frac{c_{j-1}+c_j}{2}$ and $z_2=\frac{c_{j+1}+c_j}{2}$ we have $
c_{j+1}-c_{j-1}\;=\;2(z_2-z_1)\;=\;2\Delta z$. The idea is to show
that if an impure region has probability larger than $p_0$, then the
interval $[c_{j-1},c_{j+1}]$ has probability at least $p_1$,
significantly larger than $p_0$. On the other hand, the probability of
sampling from $P$ a single center $\clust_j$ out of a total of $L$ in
an interval of length $2\Delta z$ is
$P[c_{j-1},c_{j+1}](1-P[c_{j-1},c_{j+1}])^{L-1}<(1-p_1)^{L-1}$.  If
$p_1$ and $L$ are large enough, then
$(1-p_1)^{L-1}\stackrel{def}{=}\delta_{impure}$ will be vanishingly
small. We proceed in two steps: first we find the minimum length
$\Delta z_0$ of a cluster $\clust_j$ which is impure and large. Then,
we find a lower bound $p_1$ on the probabability of any interval
$[c,\,c+2\Delta z_0]$ under the mixture distribution.  The following
assumption ensures that the purity $1-\tau$ is attainable in each
Gaussian.
\begin{assumption}\label{ass:local-purity} Let
$\gamma_{k,k'}(x)=\frac{w_{k'}\phi_{\mu_{k'},1}(x)}{w_{k}\phi_{\mu_{k},1}(x)}$
(a local purity measure). Then
\ba
\sum_{k'\neq k}\gamma_{k,k'}\left(\finv\left(
\frac{1}{2}+\frac{(1-\tau)p_0}{2w_{min}}\right)\right)
\;\leq\;\frac{\tau}{1-\tau}.
\ea
\end{assumption}

\commentout{
\begin{proposition}\label{lem:ass-local-purity} If assumption
\ref{ass:local-purity} holds, then near each Gaussian mean $\mu_k$ there is a region $\clust$ of mass $P[\clust]\geq p_0$ which has purity at least $1-\tau$.
\end{proposition}
}
The next assumption ensures that $\Delta z_0>0$, i.e it is an informative bound.
\begin{assumption} \label{ass:2dz<dmu}
$d\left(\frac{\tau p_0}{w_{min}}\right)<\frac{1}{2}\Delta$.
\end{assumption}

\begin{proposition}[Impure clusters are small
w.h.p]\label{thm:prob-impure} Let $w_1,w_2$ be the mixture weights of two
consecutive Gaussians and define $
\Delta z_0\;=\;\Delta-d\left(\frac{\tau p_0}{w_1}\right)
-d\left(\frac{\tau p_0}{w_2}\right)$,
\[ p_1 =
w_1\Phi\left(\frac{\Delta-2\Delta z_0}{2}-\frac{\ln \frac{w_1}{w_2}}{\Delta-2\Delta z_0}\right)
+w_2\Phi\left(\frac{\Delta-2\Delta z_0}{2}-\frac{\ln
\frac{w_2}{w_1}}{\Delta-2\Delta z_0}\right)
\]
and $\delta_{impure}=(1-p_1)^{L-1}$. Let $\clust^0_j,j=1,\hdots,L$ be the
regions associated with $\cenzero_{1:L}$ after step \ref{step:assign0}
of the \initalg~ algorithm. If assumptions
\ref{ass:wmin},\ref{ass:local-purity},\ref{ass:2dz<dmu} hold, then the
probability that there exists $j \in \{1,\hdots,L\}$ so that $P[\clust_j^0]\geq p_0$ and
$w_1P_1[\clust_j^0]\geq \tau P[\clust_j^0],\; w_2P_2[\clust_j^0]\geq \tau
P[\clust_j^0]$ is at most $\delta_{impure}$.  This probability is over
the random initialization of the centroids $\cenzero_{1:L}$.
\end{proposition}
To apply this proposition without knowing the values of $w_1,w_2$ one
needs to minimize the bound $p_1$ over the range $w_1,w_2>w_{min},
\,w_2+w_1\leq 1-(K-2)w_{min}$. This
minimum can be obtained numerically if the other quantities are known.

We also stress that because of the two-step approach, first minimizing
$\Delta z_0$, then $P[c,c+2\Delta z_0]$, the bound $\delta_{impure}$
obtained is not tight and could be significantly improved.

\subsection{The $(1-\tau)$-pure cluster}
\label{sec:almost-pure}

Now we focus on the clusters that have $P[\clust]>p_0$ and are
$(1-\tau)$-pure. By Proposition \ref{thm:prob-impure}, w.h.p their
centroids are the only ones which survive the thresholding in step
\ref{step:prune} of the \initalg~ algorithm. In this section we will
find bounds on the distance $|\cenunu_j-\mu_k|$ between $\clust_j$'s
center of mass and the mean of ``its'' Gaussian.

We start by listing some useful properties of the standard
Gaussian. Denote by $r(x)$ the center of mass of $[x,\infty)$ under
the truncated standard Gaussian, and by $d(t)$ the solution of
$1-\Phi(d)=t$, with $0<t<1$.
Intuitively, $d(t)$ is the cutoff
location for a tail probability of $t$. Note that any interval whose
probability under the standard normal exceeds $t$ must intersect
$[-d(t),d(t)]$. Let $a>0$ (in the following $a$ as to be thought as a small positive constant).
\begin{proposition}
(i) $r(x)$ is convex, positive and increasing for $x\geq 0$
(ii) For $w\in[2a,\infty)$ the function $d(a/w)$ is convex, positive and increasing
w.r.t $w$, and $r(d(a/w))$ is also
convex, positive and increasing.
\end{proposition}

\begin{proposition}\label{lem:distance-to-center}
Let $\clust=[z_1,z_2]$ be an interval (with $z_1,z_2$ possibly
infinite), $c$ its center of mass under the  normal
distribution $\phi_{\mu,1}$ and $P[\clust]$ its probability under the same
distribution. If $1/2\geq P[\clust]\geq p$, then $|c-\mu| \;\leq\;
r(d(p))$ and $\min \{ |z_1-\mu|,\,|z_2-\mu| \} \;\leq\; d(p)\;=\; -\finv(p)$.

\end{proposition}
The proofs are straightforward and omitted.
Define now $w_{max} =
1-(K-1)w_{min}$ the maximum possible cluster size in the mixture and
\[
R(w)= r\left[-\finv\left(\frac{(1-\tau)p_0}{w}\right)\right]
\; , \;\;
\tilr(w_1,w_2)= - \finv\!\!\!\left[\frac{\tau w_1}{(1-\tau)w_2}
+\Phi(d(\frac{(1-\tau)p_0}{w_1}\!-\!\Delta)\right]
\]

In the next proposition, we will want to assume that $\tilr\geq
0$. The following assumption is sufficient for this purpose.
\begin{assumption}\label{ass:tau}
$\frac{\tau}{w_{min}}\;\leq\;\frac{1}{2}-\Phi(-\Delta/2)$
\end{assumption}

\begin{proposition}[The $(1-\tau)$-pure cluster]\label{lem:almost-pure} Let
  cluster $\clust=[z_1,z_2]$ with  $z_2>\mu_k$, $P[\clust]\geq p_0$ and
$w_kP_k[\clust]\;\geq\;(1-\tau)P[\clust]$ for some $k$,
 with $\tau$ satisfying Assumptions \ref{ass:local-purity} and
\ref{ass:tau}. Let $c,c_{k}$ denote the
center of mass of $\clust$ under $P,P_{k}$ respectively. Then
\banum \label{eq:bound-mass-center-almost-pure-c1}
|c_k-\mu_k|\;\leq\;R(w_k)
\eanum
and, whenever $k<K$
\banum\label{eq:bound-mass-center-almost-pure-0}
z_2-\mu_k&\leq-\tilr(w_k,w_{k+1})
\,\leq\, -\tilr(w_{max},w_{min})
\eanum
\end{proposition}
\begin{proposition}[Corollary] If $c_k>\mu_k$ and $k<K$ then
\banum
c-\mu_k&\leq (1-\tau)R(w_k)+\tau(\Delta-\tilr(w_k,w_{k+1}))
\label{eq:bound-mass-center-almost-pure-1}\\
&\leq(1-\tau)R(w_{max})+\tau(\Delta-\tilr(w_{max},w_{min}))
\label{eq:bound-mass-center-almost-pure-2}\\
&\leq(1-\tau)R(w_{max})+\tau\Delta
\label{eq:bound-mass-center-almost-pure-3}\\
{\rm else}\nonumber\\
\mu_k-c&\leq R(w_k)\;\leq\;R(w_{max})
\label{eq:bound-mass-center-almost-pure-4}
\;\;\;
c-\mu_k \leq \tau(\Delta-\tilr(w_k,w_{k+1}))
\eanum
\end{proposition}
By symmetry, a similar statement involving
$\mu_{k-1},w_{k-1},\mu_k,w_k$ and $c$ holds when $z_2>\mu_k$ is
replaced by $z_1<\mu_k$. With it we have essentially shown that an
almost pure cluster which is not small cannot be too far from its
Gaussian center $\mu_k$.

{\bf Proof of Proposition \ref{lem:almost-pure}} (\ref{eq:bound-mass-center-almost-pure-c1}) follows from  Proposition \ref{lem:distance-to-center}.
Now for bounding $z_2$, in the case $k<K$.  Because $(1-\tau)P[\clust]\leq w_k$ (the
contribution of Gaussian $k$ to cluster $\clust$ cannot exceed all of $w_k$)
we have
$
P_{k+1}[C]\,\leq\,\frac{\tau P[\clust]}{w_{k+1}}\;\leq\;\frac{\tau
w_k}{(1-\tau)w_{k+1}}$ and $
P_{k+1}[C]\,=\,\Phi(z_2-\mu_{k+1})-\Phi(z_1-\mu_{k+1})
\,\geq\,\Phi(z_2-\mu_{k+1})-\Phi(c_1-\mu_{k+1})$
from which the first inequality in
(\ref{eq:bound-mass-center-almost-pure-0}) follows. The function
$\tilr$ is increasing with $w_k$ when $w_{k+1}$ constant or
$w_{k+1}={\rm constant}-w_1$, which gives the second bound. \hfill $\Box$

{\bf Proof of the corollary}
First note that we can safely assume $z_1\geq\mu_k$. If
the result holds for this case, then it is  easy to see that having
$z_1<\mu_k$ only brings the center of mass $c$ closer to $\mu_k$.
\banum \label{eq:c-weighted-sum}
c=\frac{w_kP_k[C]c_k+\sum_{k'\neq k}w_{k'}P_{k'}[C]c_{k'}}{P[\clust]}
\leq(1-\tau)c_k+\tau z_2
\eanum
Now
(\ref{eq:bound-mass-center-almost-pure-1},\ref{eq:bound-mass-center-almost-pure-2})
follow from Proposition \ref{lem:almost-pure}.  For
(\ref{eq:bound-mass-center-almost-pure-3}) Assumption \ref{ass:tau}
assures that $\tilr\geq 0$. As a consequence, this bound
is convex in $w_k$.
If $k=1$ and $c_1\leq \mu_1$, or $k=K$ and $c_K>\mu_K$ then the second term in the sum
(\ref{eq:c-weighted-sum}) pulls $c_1$ in the direction of $\mu_1$
(respectively $c_K$ in the direction of $\mu_K$) and we can get the
tighter bounds (\ref{eq:bound-mass-center-almost-pure-4}). \hfill $\Box$

In conclusion, we have shown now that if the unpruned center $c$ ``belongs'' to Gaussian $k$, then
\ba
c\,\in\,\tila_k(w_k)\;=\;[\,\mu_k-R^-_\tau(w_k),\,\mu_k+R^+_\tau(w_k)\,]
\ea
whith $R^-_\tau(w_k)=(1-\tau)R(w_k)+\tau(\mu_k-\mu_{k-1})$,
$R^+_\tau(w_k)=(1-\tau)R(w_k)+\tau(\mu_{k+1}-\mu_k)$,
$R^-_\tau(w_1)=R(w_1)$, and $R^+_\tau(w_K)=R(w_K)$.

\subsection{Step \ref{step:HS} of \initalg. Selecting the centers by the {\sc
MinDiam} heuristic}
\label{sec:step4}

From Section \ref{sec:impure} we know that w.h.p all centroids
unpruned at this stage are $(1-\tau)$ pure. We want to ensure that
after the selection in step \ref{step:HS} each Gaussian has at least
one $\cenunu_j$ near its center. For this, it is sufficient that the
regions $\tila_k$ are disjoint, i.e
\commentout{ By proposition \ref{lem:almost-pure}, we
know that such a centroid is no further than the r.h.s of
(\ref{eq:bound-mass-center-almost-pure-2}) from its Gaussian's center.
Each Gaussian has at least one centroid if  For this we need,}
\hspace{-2em}\ba
(\mu_{k+1}-\mu_k)-(R^+_\tau(w_k)+R^-_\tau(w_{k+1}))
&>
R^-_\tau(w_k)+R^+_\tau(w_k)\\
(\mu_{k+1}-\mu_k)-(R^+_\tau(w_k)+R^-_\tau(w_{k+1}))
&>
R^-_\tau(w_{k+1})+R^+_\tau(w_{k+1})
\ea
for all $k$.  Replacing $R^\pm_\tau(w_k)$ with their definitions and
optimizing over all possible $w_{1:K}\geq w_{min}$ and for all $\Delta
\mu \leq \mu_{k+1}-\mu_k\leq \Delta_{max}$ produces
\[
\tila_k\;=\;[\mu_k\pm(1-\tau)R(w_{max})\pm\tau\Delta_{max}]
\]
and
\begin{assumption}\label{ass:bounds-for-hs}
$(1-3\tau)\Delta-\tau \Delta_{max}\;>\;[3R(w_{max})+R(w_{min})](1-\tau).$
\end{assumption}

\commentout{
\subsection{Back to Theorem \ref{thm:initialization0}}
Combining the results from above, we can now complete the statement of
\ref{thm:initialization0}.
\begin{theorem}[\initalg~ Initialization]\label{thm:initialization} Let
$f=\sum_1^Kw_k\phi_{\mu_k,1}$ be a mixture of $K$ Gaussians with
centers $\mu_{1:K},\,\mu_k\leq\mu_{k+1}$, and unit variance. Let
$w_{min},\Delta, \Delta_{max},\tau,\delta_{miss}$ be so that assumptions
\ref{ass:wmin}, \ref{ass:local-purity}, \ref{ass:2dz<dmu}, \ref{ass:tau},
\ref{ass:bounds-for-hs} hold.  If we run Algorithm \initalg~ with any
$2 \leq K' \leq {1}/{w_{min}}$, then with probability
$1-2\delta_{miss}-\delta_{impure}$ over the initialization there exist
$K$ disjoint neighborhoods $\tilde{A}_k$, one for each true mean
$\mu_k$, so that all $K'$ centers $\cenunu_{k'}$ are contained in
$\bigcup_k \tila_k$ and statements {\bf T6$\mathbf{=}$, T6$\mathbf{<}$, T6$\mathbf{>}$} hold.
\end{theorem}
}

\section{Simulations} \label{sec-simulations}

\begin{figure*}
\begin{center}
\includegraphics[width=0.20\textwidth]{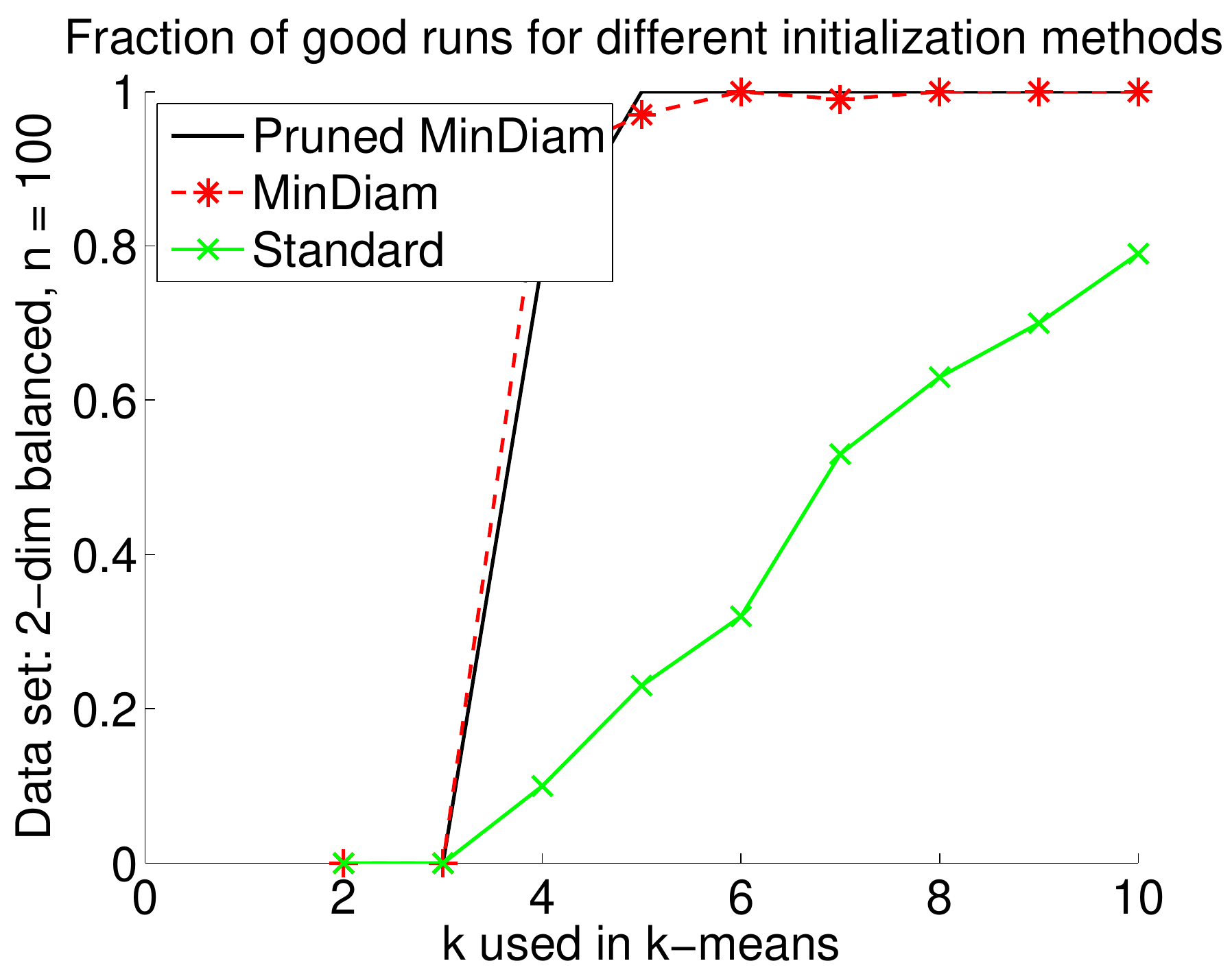}\hfill
\includegraphics[width=0.20\textwidth]{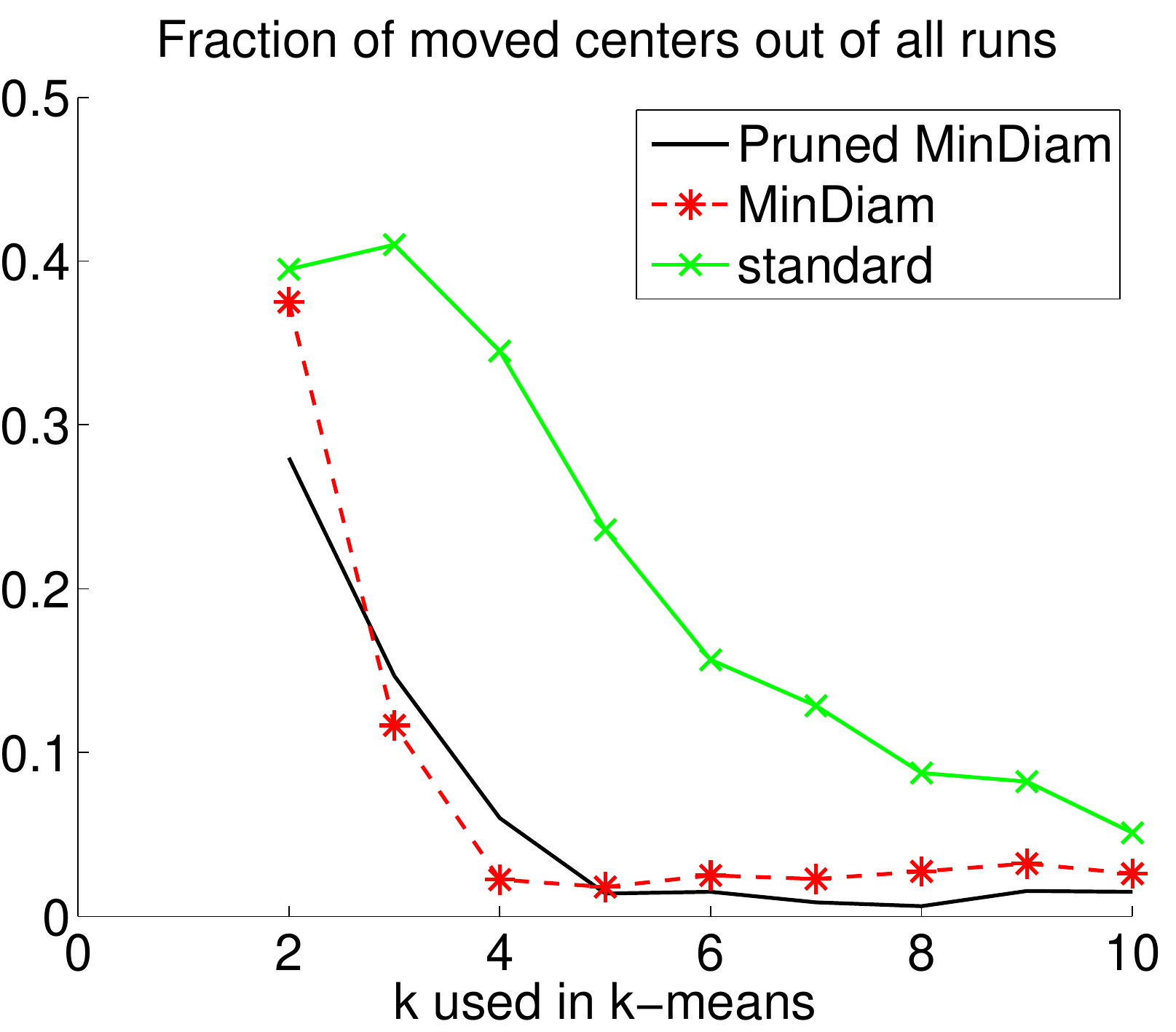}\hfill
\includegraphics[width=0.20\textwidth]{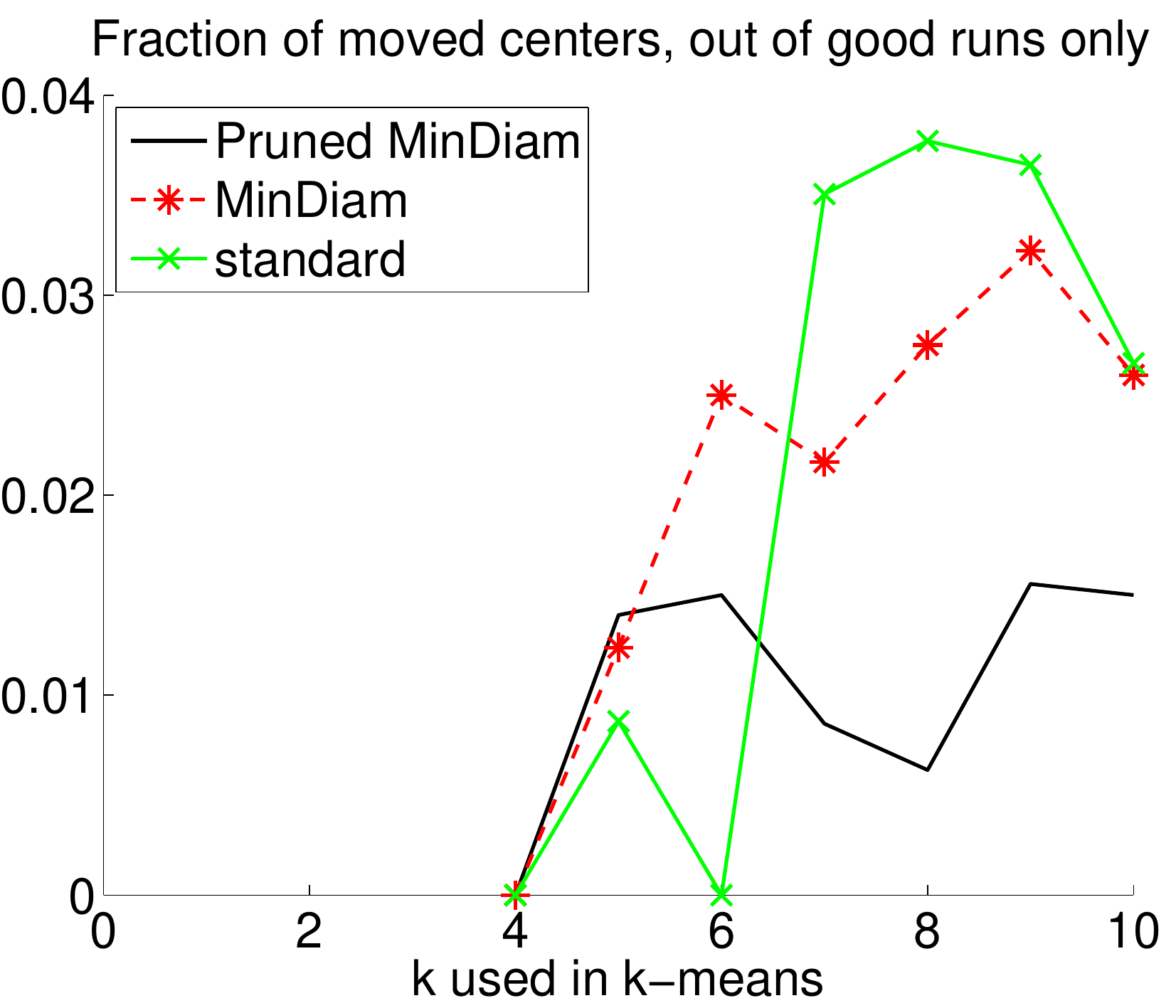}\hfill
\includegraphics[width=0.20\textwidth]{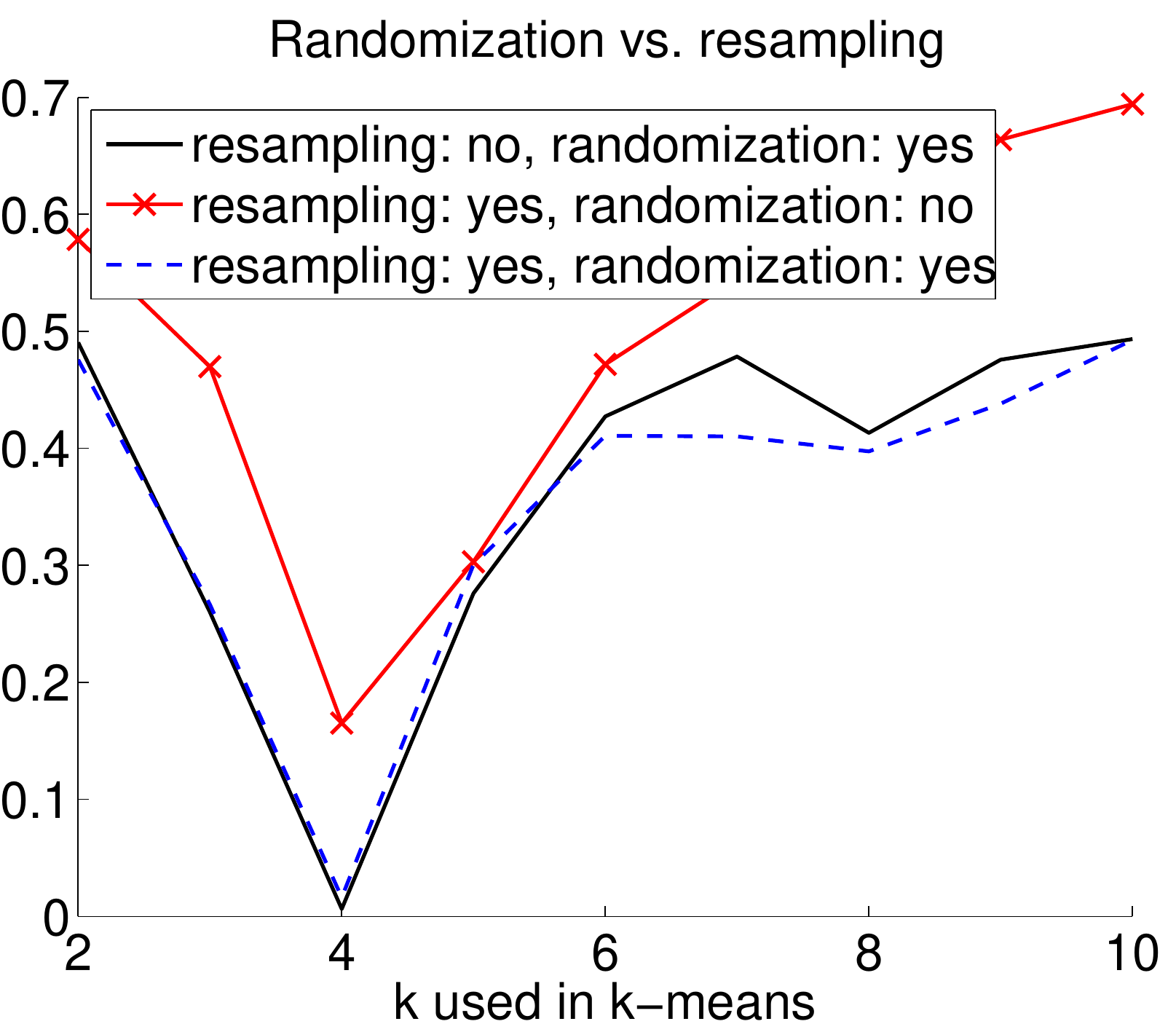}\\
\includegraphics[width=0.20\textwidth]{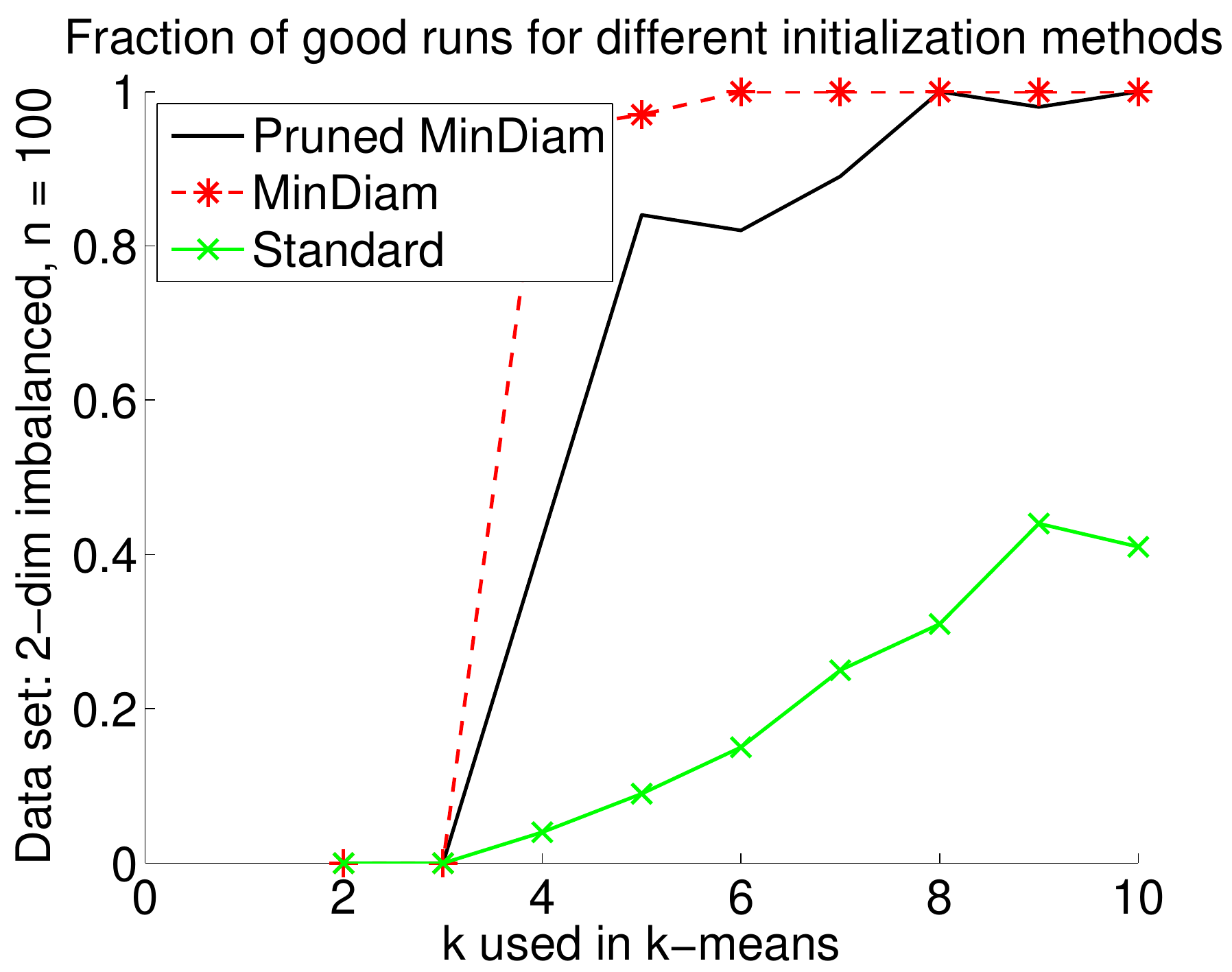}\hfill
\includegraphics[width=0.20\textwidth]{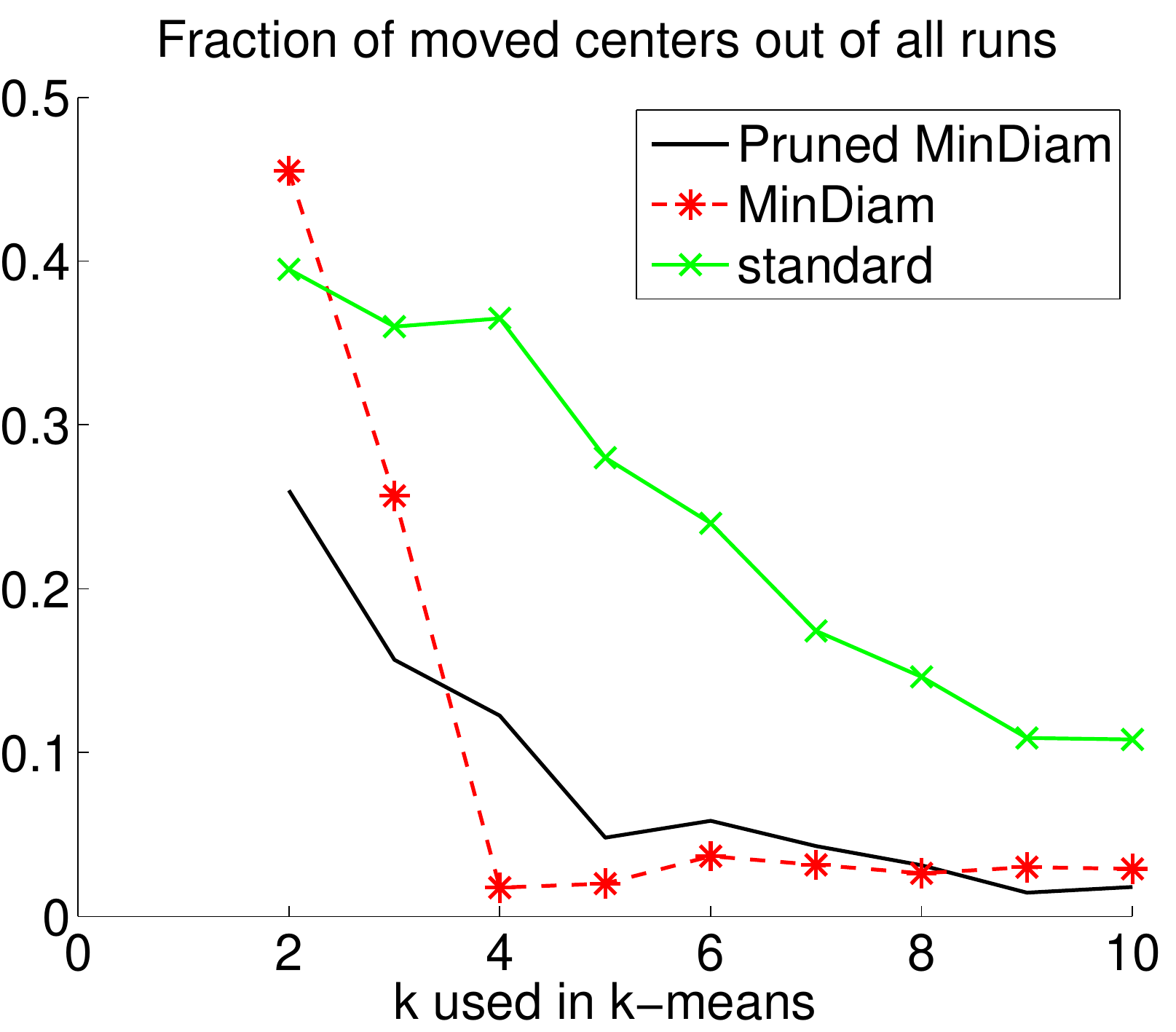}\hfill
\includegraphics[width=0.20\textwidth]{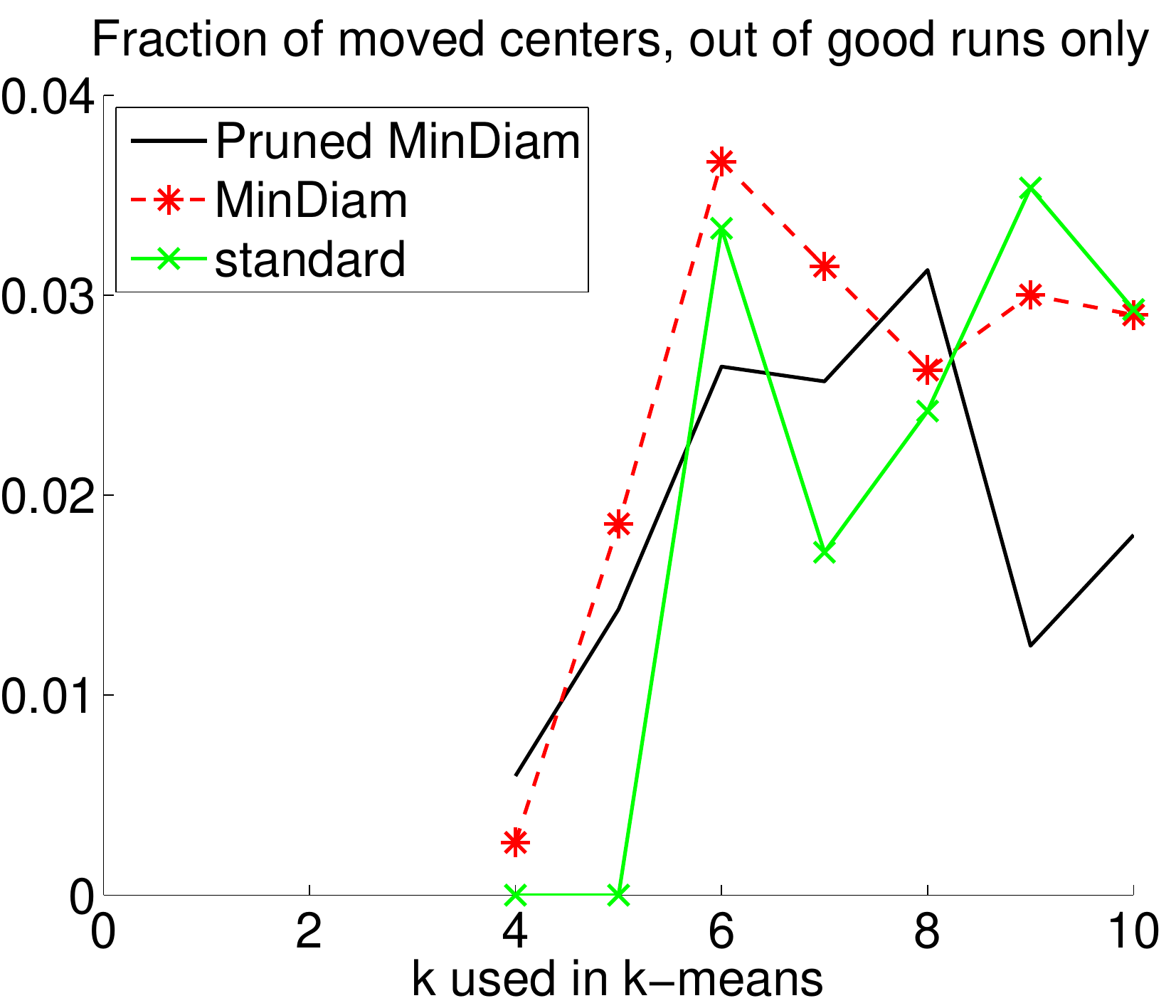}\hfill
\includegraphics[width=0.20\textwidth]{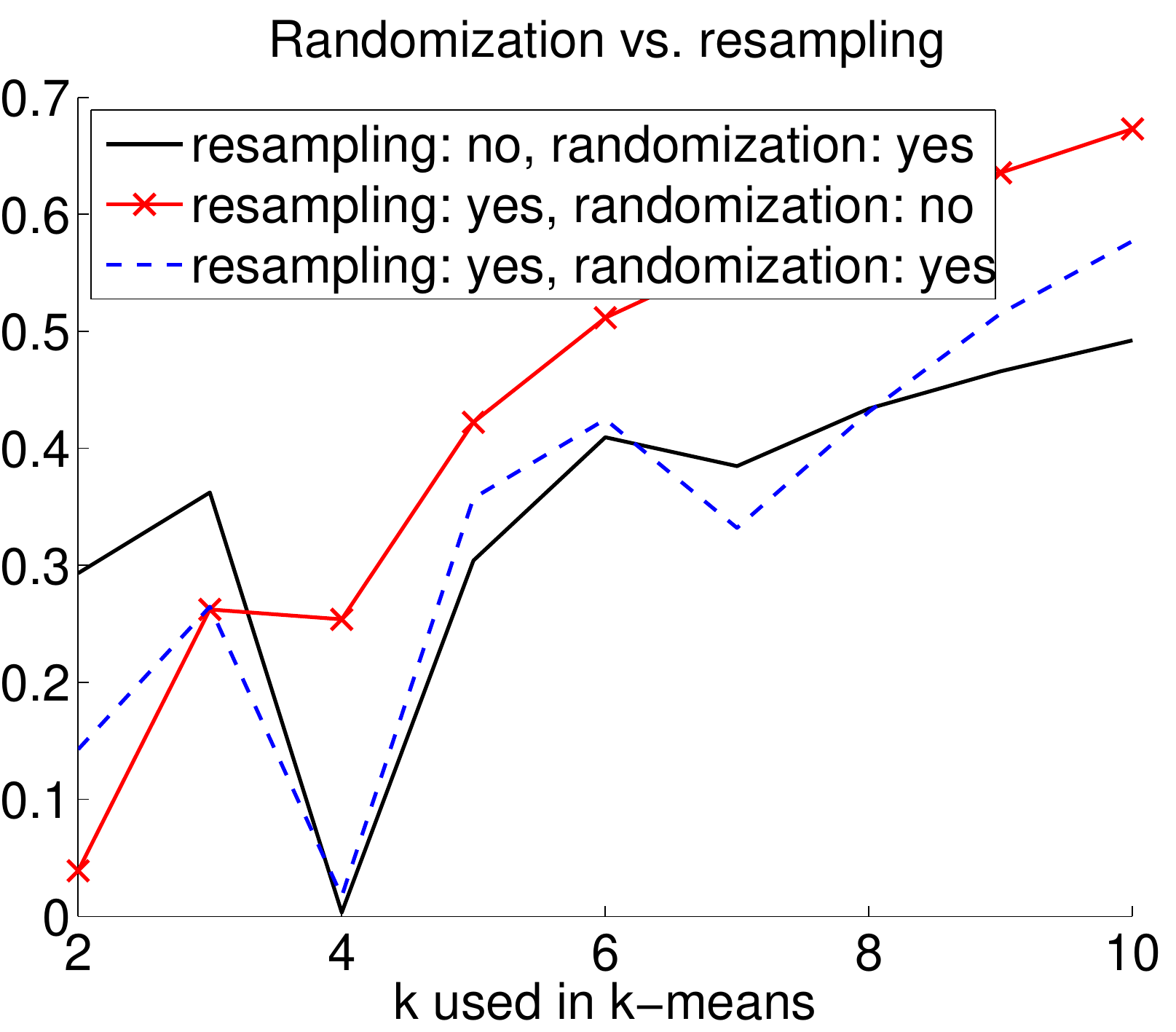}\\
\includegraphics[width=0.20\textwidth]{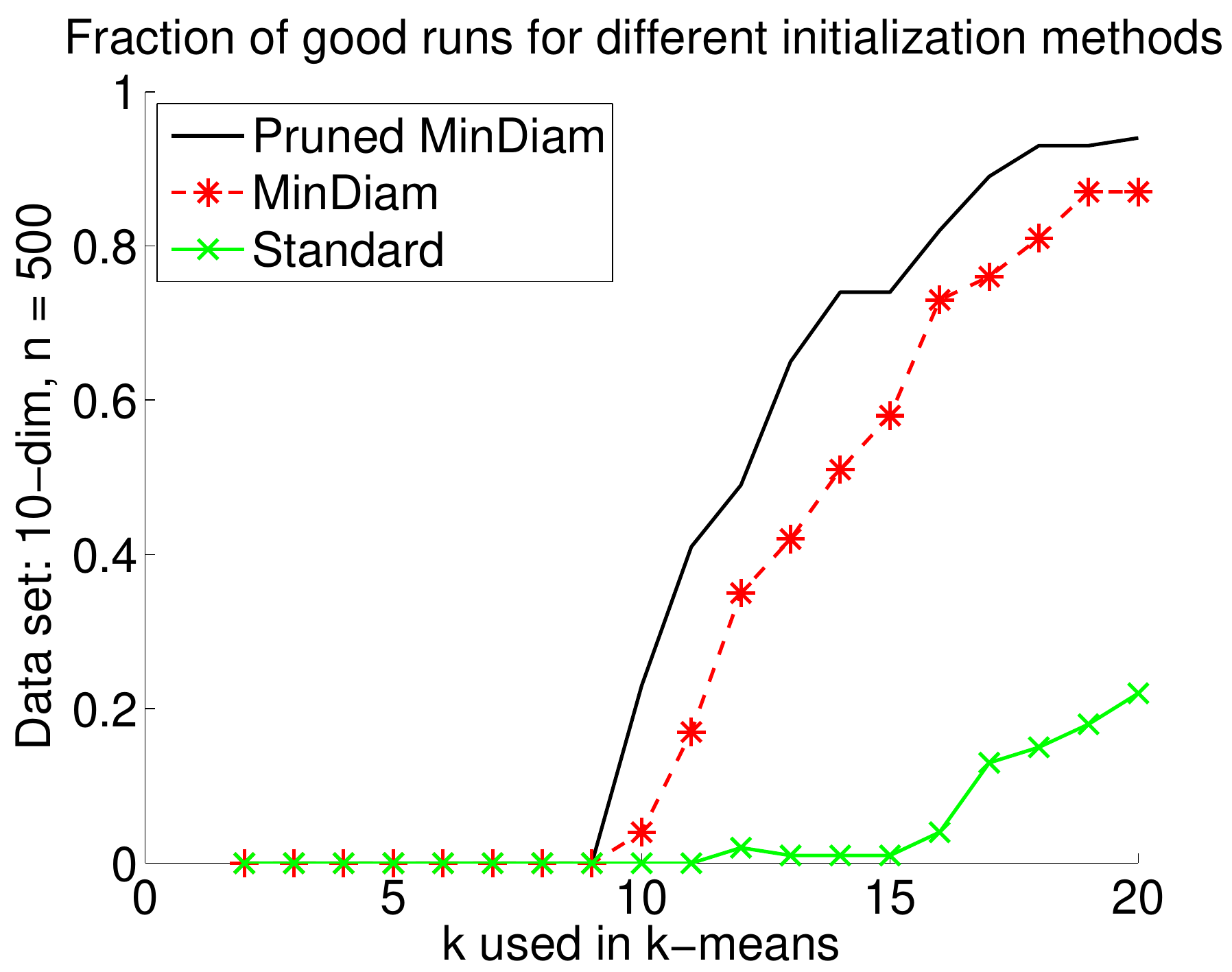}\hfill
\includegraphics[width=0.20\textwidth]{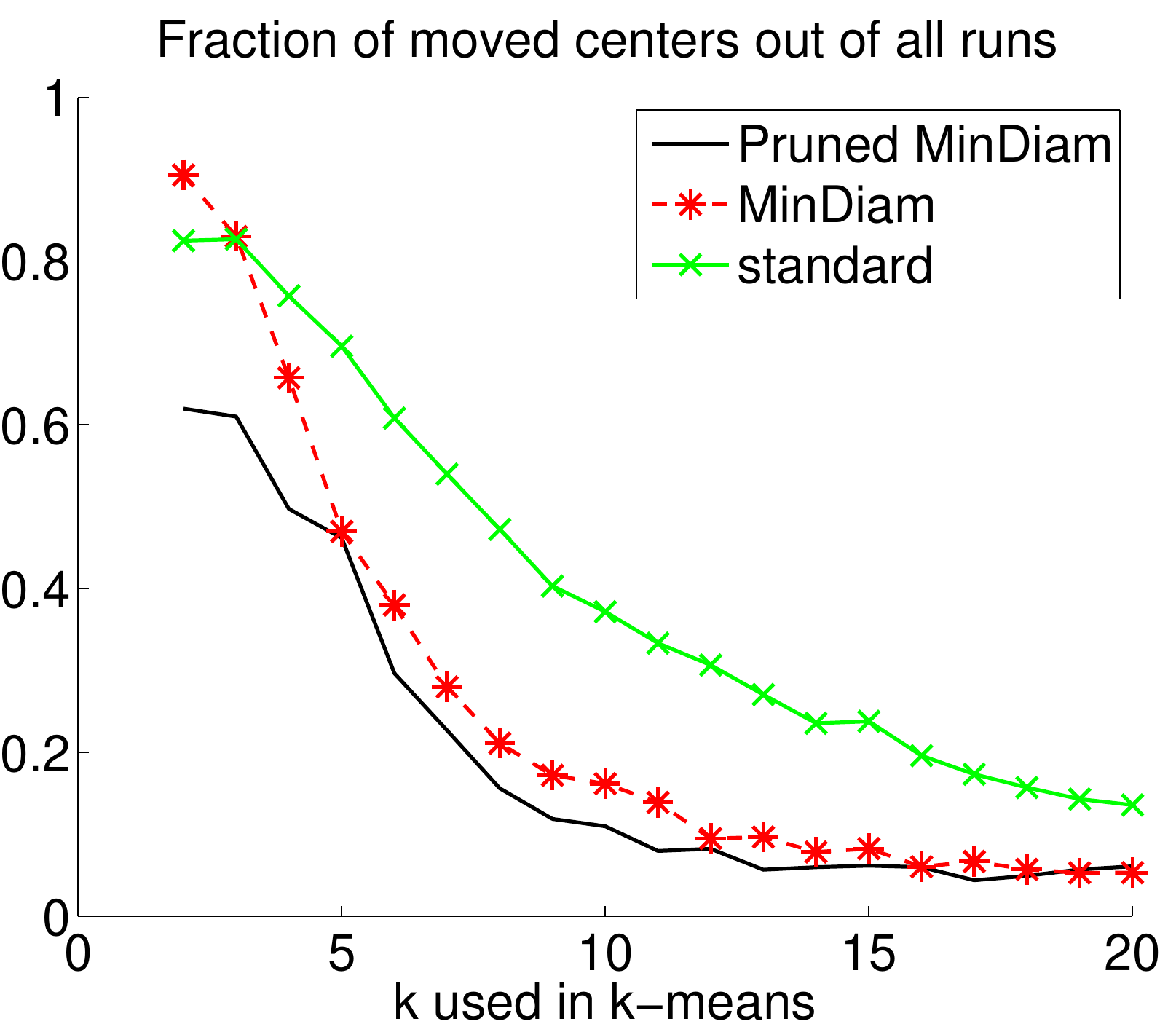}\hfill
\includegraphics[width=0.20\textwidth]{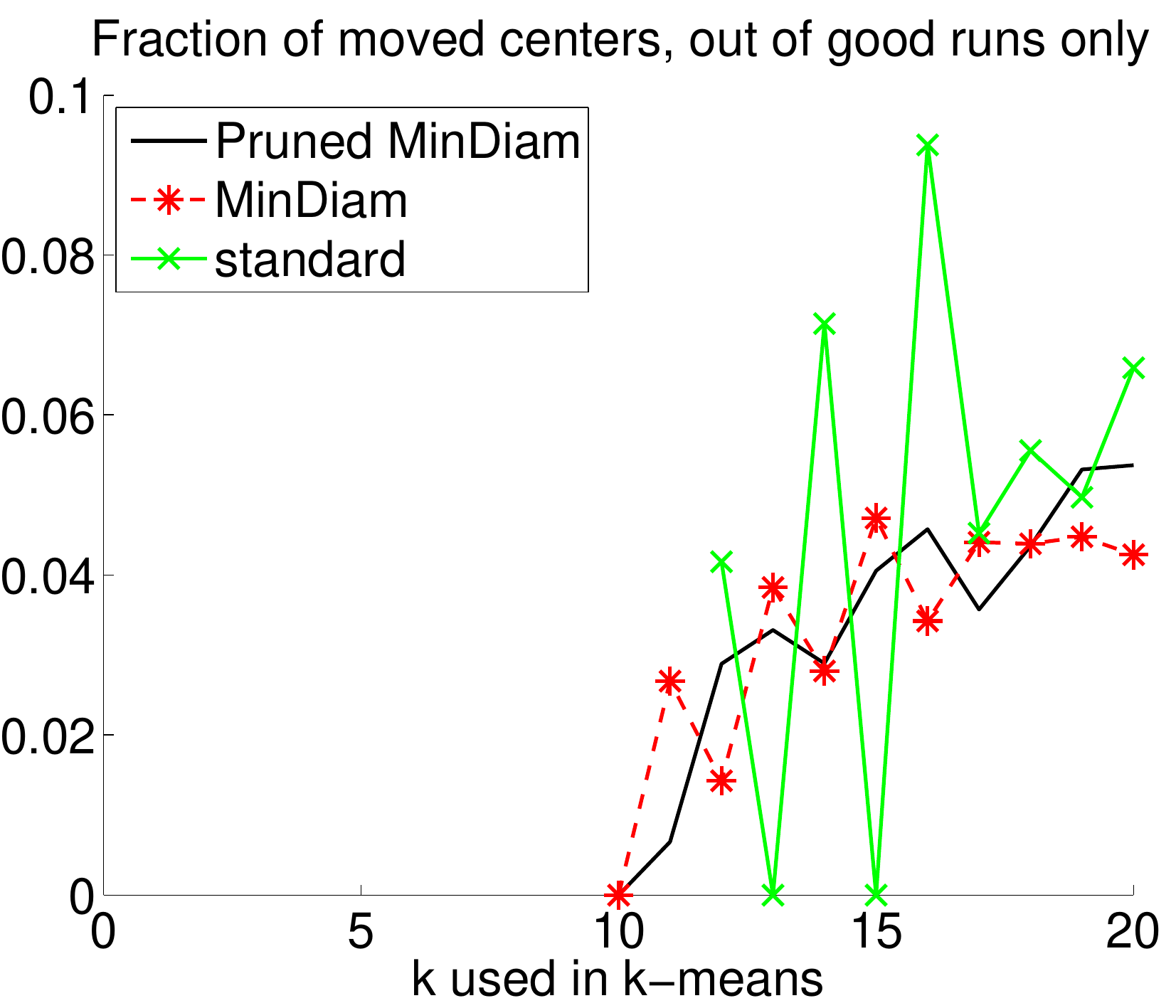}\hfill
\includegraphics[width=0.20\textwidth]{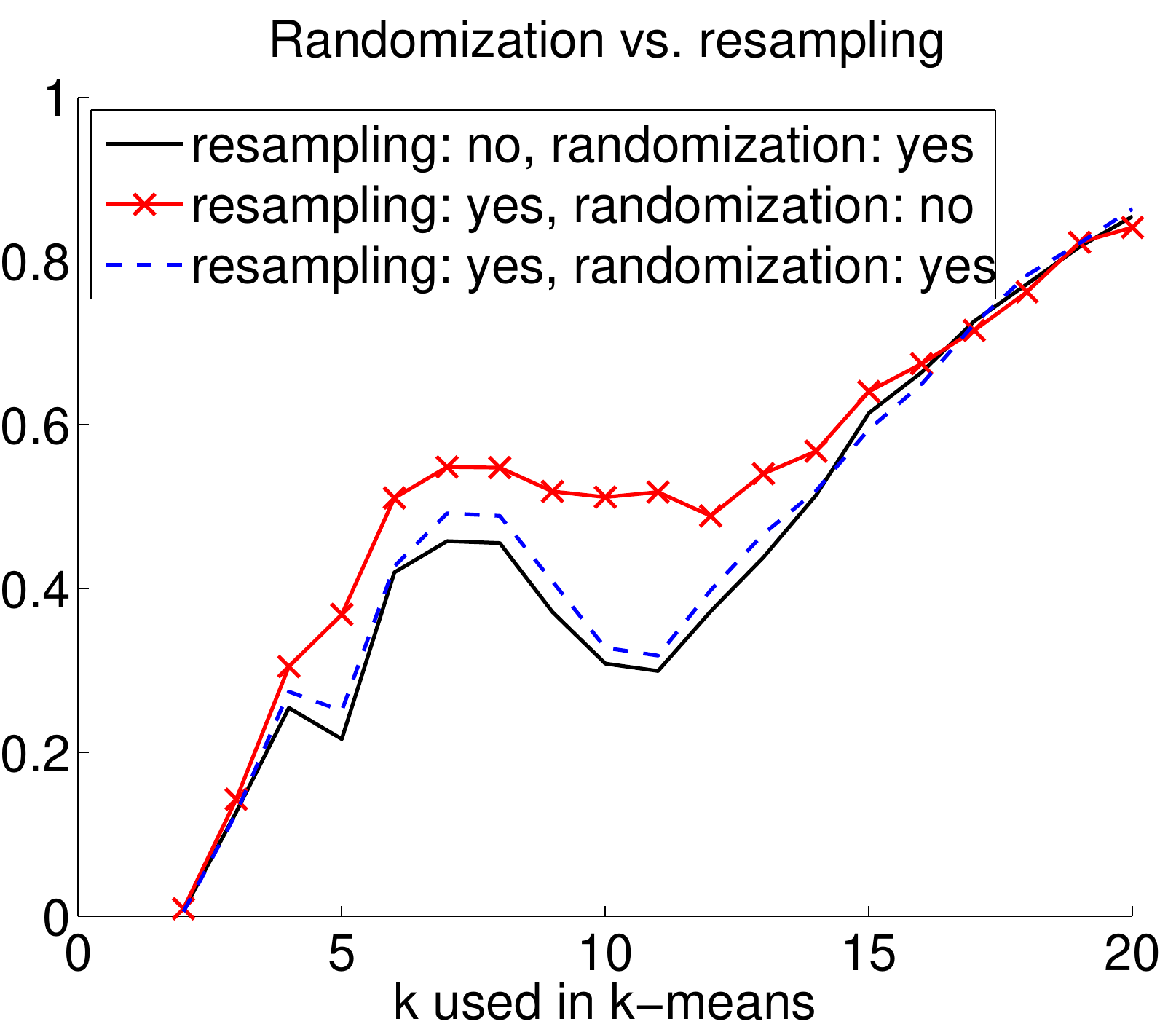}\\
\end{center}
\caption{Simulation results. First row: data set ``two dim four
  balanced clusters''. Second row: data set ``two dim four
  imbalanced clusters''. Third row: data set ``ten dim ten clusters''
  (see text for details). }
\label{fig-simulations}
\end{figure*}

In this section we test our conjecture in practice and run some simulations to emphasize the different theoretical results of the previous sections. We also investigate whether it is necessary to look at the stability of $k$-means with respect to the random drawing of the data set. In the following when we refer to randomization we mean with respect to the initialization while the resampling corresponds to the random drawing of the data set.
\newline

{\bf Setup of the experiments. } As distributions we consider mixtures
of Gaussians in one, two, and ten dimensions. Each mixture consists of
several, reasonably well separated clusters. We report the results on
three such data sets:

\blobb{``Two dim four balanced clusters'':
Mixture of four Gaussians in $\R^2$ with means
$(-3.3)$, $(0,0)$, $(3, 3)$, $(3, -3)$; the covariance matrix of all
clusters is diagonal with entries $0.2$ and $1$ on the diagonal; the
mixing coefficients are uniform, that is all clusters have the same
weight.
}

\blobb{``Two dim four imbalanced clusters'': As above, but with mixing coefficients  $0.1, 0.5, 0.3,
0.1$.
}

\blobb{``Ten dim ten clusters'': Mixture of ten Gaussians in $\R^{10}$ with means $(i, 0, 0, ...)$ for $i=1, ..., 10$.  All Gaussians are spherical with variance 0.05 and mixing coefficients are uniform.
}

As clustering algorithm we use the standard $k$-means algorithm with
the following initializations:

\blobb{Standard initialization: randomly pick $K'$ data points. }
\blobb{{\sc MinDiam} initialization, coincides with Step 5 in Fig.~\ref{fig:init-alg}.}
\blobb{\initalg initialization, as analyzed in Section \ref{sec-init}
  (see Fig.~\ref{fig:init-alg}}
\blobb{Deterministic initialization: $K'$ fixed points sampled from the
  distribution. }

For a range of parameters $K' \in \{2, ..., 10\}$ we compute the
clustering stability by the following protocols:

\blobb{Randomization, no resampling: We draw once a data set of
  $n=100$ points from the distribution. Then we run the
  $k$-means algorithm with different initializations. }

\blobb{Resampling, no randomization: We fix a set of deterministic
  starting points (by drawing them once from the underlying
  distribution). Then we draw $100$ data sets of size $n = 100$ from the
  underlying distribution, and run $k$-means with the deterministic
  starting points on these data sets. }

\blobb{Resampling and randomization: we combine the  two previous
  approaches. }

Then we compute the stability with respect to the minimal matching
distance between the clusters. Each experiment was repeated 100 times,
we always report
the mean values over those
repetitions. \\

Note that all experiments were also conducted with different data set sizes
($n=50, 100, 500$), stability was computed with and without
normalization (we used the normalization suggested in
\citealp{LanRotBraBuh04}), and the $k$-means algorithm was used with and
without restarts. All those variations did not significantly effect
the
outcome, hence we omit the plots. \\

{\bf Results. } First we evaluate the effect of the different
initializations. To this end, we count how many initializations were
``good initializations'' in the sense that each true cluster contains
at least one initial center. In all experiments we consistently
observe that both the pruned and non-pruned min diameter heuristic already
achieve
many good runs if $K'$ coincides with $K$ or is
only slightly larger than the true $K$ (of course, good runs
  cannot occur for $K' < K$).  The standard random initialization does
  not achieve the same performance. See Figure
  \ref{fig-simulations}, first column.\\

  Second, we record how often %
  it was the case that initial cluster centers cross
  cluster borders. We can see in Figure~\ref{fig-simulations} (second
  column) that this behavior is strongly correlated with the number of
  ``good initializations''. Namely, for initialization methods which
  achieve a high number of good initializations the fraction of
  centers which cross cluster borders is very low.  Moreover, one can see in the third column
  of Figure \ref{fig-simulations} that centers usually do not cross
  cluster borders if the initialization was a good one. This
  coincides with our theoretical results. \\

  Finally, we compare the different protocols for computing the
  stability: using randomization but no resampling, using resampling
  but no randomization, and using both randomization and resampling,
  cf. right most plots in Figure~\ref{fig-simulations}. In simple data
  sets, all three protocols have very similar performance, see for
  example the first row in Figure \ref{fig-simulations}.  That is, the
  stability values computed on the basis of resampling behave very
  similar to the ones computed on the basis of randomization, and all
  three methods clearly detect the correct number of
  clusters. Combining randomization and resampling does not give any
  advantage. However, on the more difficult data sets (the imbalanced
  one and the 10-dimensional one), we can see that resampling without
  randomization performs worse than the two protocols with
  randomization (second and third row of Figure
  \ref{fig-simulations}). While the two protocols using randomization
  have a clear minimum around the correct number of clusters,
  stability based on resampling alone fails to achieve this. We never
  observed the opposite effect in any of our simulations (we ran many
  more experiments than reported in this paper). This shows, as we had
  hoped, that randomization plays an important role for clustering
  stability, and in certain settings can
  achieve better results than resampling alone. \\

  Finally, in the experiments above we ran the $k$-means algorithm in
  two modes: with restarts, where the algorithm is started 50 times
  and only the best solution is kept; and without restarts. The
  results did not differ much (above we report the results without
  restarts). This means that in practice, for stability based
  parameter selection one can save computing time by simply running
  $k$-means without restarting it many times (as is usually done in
  practice). From our theory we had even expected that running
  $k$-means without restarts achieves better results than with
  restarts. We thought that many restarts diminish the effect of
  exploring local optima, and thus induce more stability than ``is
  there''. But the experiments did not corroborate this intuition.

\section{Conclusions and outlook}

Previous theoretical work on model selection based on the stability of
the $k$-means algorithm has assumed an ``ideal $k$-means algorithm''
which always ends in the global optimum of the objective function. The
focus was to explain how the random drawing of sample points
influences the positions of the final centers and thus the stability
of the clustering. This analysis explicitly excluded the
question when and how the $k$-means algorithm ends in different local
optima. In particular, this means that these results only have a
limited relevance for the actual $k$-means algorithm as used in
practice. \\

In this paper we study the actual $k$-means algorithm. We have
shown that the initialization strongly influences the $k$-means
clustering results.  We also show that if one uses a ``good''
initialization scheme, then the $k$-means algorithm is stable if it is
initialized with the correct number of centers, and instable if it is
initialized with too many centers. Even though we have only proved
these results in a simple setting so far, we are convinced that the same
mechanism also holds in a more general setting. \\

These results are a first step towards explaining
why the selection of the number of clusters based on
clustering stability is so successful in practice
\citet{LanRotBraBuh04}. From this practical point of view, our
results suggest that introducing randomness by the initialization may
be sufficient for an effective model selection algorithm. Another
aspect highlighted by this work is that the situations $\kalg<\ktrue$ and
$\kalg>\ktrue$ may represent two distinct regimes for clustering, that
require separate concepts and methods to be analyzed.\\

The main conceptual insight in the first part of the paper is the
configurations idea described in the beginning. With this idea we
indirectly characterize the ``regions of attraction'' of different
local optima of the $k$-means objective function. To our knowledge,
this is the first such characterization in the vast literature of
$k$-means. \\

In the second part of the paper we study an initialization scheme for
the $k$-means algorithm. Our intention is not to come up with a new
scheme, but to show that a scheme already in use is ``good'' in the
sense that it tends to put initial centers in different clusters. It
is important to realize that such a property does not hold for the
widely used uniform random initialization.

On the technical side, most of the proofs and proof ideas in this
section are novel. In very broad terms, our analysis is reminiscent to that of
\citet{dasgupta:07}. One reason we needed new proof techniques lie
partly in the fact that we analyze one-dimensional Gaussians, whose
concentration properties differ qualitatively from those of high
dimensional Gaussians. We loose some of the advantages high
dimensionality confers. A second major difference is that $k$-means
behaves qualitatively differently from EM whenever more than one
Gaussian is involved. While EM weights a point ``belonging'' to a
cluster by its distance to the cluster center, to the effect that far
away points have a vanishing influence on a center $c_j$, this is not
true for $k$-means. A far-away point can have a significative
influence on the center of mass $c_j$, precisely because of the
leverage given by the large distance. In this sense, $k$-means is a
more brittle algorithm than EM, is less predictible and harder to
analyze. In order to deal with this problem we ``eliminated'' impure
clusters in Section \ref{sec:impure}. Third, while \citet{dasgupta:07}
is concerned with finding the correct centers when $\ktrue$ is known,
our analysis carries over to the regime when $\kalg$ is too large,
which is qualitatively very different of the former.

Of course many initialization schemes have been suggested and analyzed
in the literature for $k$-means (for examples see
\citealp{OstRabSchSwa06,ArtVas07}). However, these papers analyze the
{\em clustering cost} obtained with their initialization, not the
positions of the initial centers. \\

\small
\bibliography{general_bib,ules_publications,mmp_tmp}
\end{document}